\documentclass[sn-mathphys,Numbered]{sn-jnl}

\usepackage{graphicx}%
\usepackage{multirow}%
\usepackage{amsmath,amssymb,amsfonts}%
\usepackage{amsthm}%
\usepackage{mathrsfs}%
\usepackage[title]{appendix}%
\usepackage{xcolor}%
\usepackage{textcomp}%
\usepackage{manyfoot}%
\usepackage{booktabs}%
\usepackage{algorithm}%
\usepackage{algorithmicx}%
\usepackage{algpseudocode}%
\usepackage{listings}%

\usepackage{booktabs}

\usepackage{subcaption}

\usepackage{hyperref}
\usepackage{url}
\usepackage{graphicx} 
\usepackage{placeins}
\usepackage{cleveref}
\usepackage{tabularx}

\DeclareMathOperator*{\argmin}{arg\,min}

\theoremstyle{thmstyleone}%
\newtheorem{theorem}{Theorem}
\newtheorem{proposition}[theorem]{Proposition}%

\theoremstyle{thmstyletwo}%

\theoremstyle{thmstylethree}%

\begin{document}

\title[Adjusting Regression Models for Conditional Uncertainty Calibration]{Adjusting Regression Models for Conditional Uncertainty Calibration}


\author*[1]{\fnm{Ruijiang} \sur{Gao}}\email{ruijiang.gao@utdallas.edu}

\author[2]{\fnm{Mingzhang} \sur{Yin}}\email{mingzhang.yin@warrington.ufl.edu}

\author[3]{\fnm{James} \sur{McInerney}}\email{jmcinerney@netflix.com}

\author[3,4]{\fnm{Nathan} \sur{Kallus}}\email{kallus@cornell.edu}

\affil*[1]{\orgdiv{Naveen Jindal School of Management }, \orgname{University of Texas at Dallas}, \orgaddress{Dallas, Texas}}

\affil[2]{\orgdiv{Department of Marketing}, \orgname{University of Florida}, \orgaddress{Gainesville, Florida}}

\affil[3]{\orgname{Netflix}, \orgaddress{Los Gatos, California}}

\affil[4]{\orgname{Cornell Tech, Cornell University}, \orgaddress{New York, New York}}


\abstract{Conformal Prediction methods have finite-sample distribution-free marginal coverage guarantees. However, they generally do not offer conditional coverage guarantees, which can be important for high-stakes decisions. In this paper, we propose a novel algorithm to train a regression function to improve the conditional coverage after applying the split conformal prediction procedure. We establish an upper bound for the miscoverage gap between the conditional coverage and the nominal coverage rate and propose an end-to-end algorithm to control this upper bound. We demonstrate the efficacy of our method empirically on synthetic and real-world datasets.}




\maketitle

\section{Introduction}\label{sec1}

A central challenge in supervised machine learning (ML) is the estimation of a target variable $Y \in \mathcal{Y}$ based on a vector of inputs $X$. This issue involves creating a predictive function $f(Y | X)$, which is constructed based on a dataset $\mathcal{D} = \{(X_i, Y_i)\}_{i=1}^N$, drawn independently and identically distributed (i.i.d.) from an unknown distribution $P(X,Y)$. This function is then utilized to estimate the target $Y_{N+1}$ for a new data point with inputs $X_{N+1}$. Even though machine learning typically produces a point estimate of $Y$, predictive inference focuses on a more reliable prediction. It aims to develop a \textit{predictive set} that is probable to include the target that has not been observed yet \citep{geisser2017predictive}.

Specifically, \textit{conformal prediction} is a branch of predictive inference that creates algorithms to achieve calibrated coverage probabilities \citep{papadopoulos2002inductive,vovk2005algorithmic}. Under the assumption that data pairs $(X_i, Y_i)$ are i.i.d. from a population distribution $\mathbb{P}(X,Y)$, a conformal prediction algorithm proposes a set $C_{\alpha}(X)$ that satisfies 

\begin{align}
\mathbb{P}_{X,Y}(Y \in \hat{C}_{\alpha}(X)) \geq 1-\alpha.
\label{eq:coverage}
\end{align}
Here, $\alpha \in [0,1]$ is a pre-specified miscoverage rate and $\hat{C}_{\alpha}(X) \subset \mathcal{Y}$ is the predictive set. A set fulfilling \Cref{eq:coverage} is called a \textit{valid} predictive set. 

Conformal prediction methods have two desired properties: 1) they do not make distributional assumptions on the underlying data-generating process, and 2) they are valid with finite-samples. However, conformal prediction algorithms can only satisfy the marginal coverage in \Cref{eq:coverage} and generally do not offer conditional coverage $\mathbb{P}(Y \in \hat{C}_{\alpha}(X)|X=x) = 1-\alpha$, a stronger condition than the marginal coverage. 

The difference between marginal and conditional coverage can be detrimental for high-stake decision making tasks since marginal coverage does not guarantee predictive sets with valid coverage for certain subgroups, especially for rare events or minorities. For example, the predictive set can be marginally valid averaged over all populations but significantly undercover end-stage patients in survival year estimation, or misestimate the risk of minorities in tasks like recidivism prediction. More importantly, oftentimes we do not have prior knowledge that what subpopulations may be at stake, or do not have access to the sensitive attribute information \cite{chen2019fairness}, which makes the conditional coverage guarantee much desired in practice. 

Existing conformal prediction methods have considered modifying the calibration steps \cite{Guan2022-bi,lei2014distribution,han2022split} or using the quantile estimators \cite{romano2019conformalized,Sesia2021-ei,Chernozhukov2021DistributionalCP} to improve the conditional coverage of conformal prediction. However, these methods all work with a given predictive model in the training stage and mainly focus on improving the calibration step afterward. 
Unlike these approaches, we propose to optimize the predictive function in the training step given any differentiable non-conformity scores. Our approach follows the split conformal prediction framework, and we propose a novel objective to minimize the miscoverage rate and build the connection with the Kolmogorov–Smirnov distance. 

Our paper makes the following contributions:

\begin{itemize}
    \item We propose a new method to improve the conditional coverage by optimizing the predictive function with regularized  Kolmogorov–Smirnov distance between the marginal and conditional non-conformity score distributions. 
    \item Theoretically, we establish the connection between the proposed KS regularization and the conditional coverage objective. 
    \item We empirically validate the effectiveness and advantages of our proposed approaches using synthetic and real-world data. 
\end{itemize}

\section{Related Work}
\noindent \textbf{Conformal Prediction.~~}
We extend the existing body of research on conformal prediction in regression problems \cite{lei2018distribution,romano2019conformalized,Sesia2021-ei,Chernozhukov2021DistributionalCP,izbicki2020flexible,han2022split}. Many of these approaches traditionally depend on a \textit{fixed} predictive function and focuses on designing better non-conformity score to improve conditional coverage or sharper intervals. Our method focuses on optimizing the predictive function for a given differentiable non-conformity score to improve the conditional coverage. 
In more recent studies, various methods have been introduced to enhance model training with the goal of reducing the size of prediction sets in classification problems \cite{bellotti2021optimized,stutz2021learning}, or focusing on lower-dimensional hyperparameters rather than directing the training of all model parameters \cite{chen2021learning,colombo2020training,yang2021finite}. Similar to ours, \cite{einbinder2022training} proposes a differentiable objective to improve the conditional coverage of deep learning classifiers while our paper focuses on regression tasks and we propose a novel Kolmogorov–Smirnov distance-based objective which is a sufficient condition for achieving nominal conditional coverage rate.  \\

\noindent \textbf{Approximate Conditional Coverage.~~}In an early work, \cite{vovk2012conditional} introduces a number of variants of conditional validity and achieves them using a modified inductive conformal prediction. Recently, a series of works aimed to improve the imbalanced coverage of conformal prediction. 
Motivated by fairness concerns, \cite{romano2020malice} proposed the equalized coverage method that has guaranteed coverage conditional on a group index.  The coverage, however, is only guaranteed for the subgroups sharing the same value of a pre-specific sensitive attribute.  \cite{feldman2021improving} designed a regularization to encourage the independence of a coverage indicator of a miscoverage event and the predictive interval length. The regularization is a necessary but not sufficient condition for valid conditional coverage, and its effectiveness hinges on empirical validations. Noticing that the conditional coverage is equivalent to a set of moment conditions to hold for all measurable functions, \cite{gibbs2023conformal} proposes a type of conditional coverage given a  class of covariate shifts. It can provide conditional coverage over groups and over multiple pre-specified shifts,
while the applications are mostly designed for group-conditional coverage with pre-specified groups and for coverage under covariate shifts with given tilting functions. In contrast, our proposed method does not require the variables to be conditioned on are known a priori.

\section{Problem Statement}

Consider i.i.d. pairs of covariates $X_i$ and a target variable $Y_i$, i.e. $\mathcal{D} = \{(X_i, Y_i)\}_{i=1}^N$, 
from an underlying distribution $P$. We observe data $\mathcal{D}$ and the covariates $X_{N+1}$ of a new data point. 
The non-conformity score function $V: \mathcal{X}\times \mathcal{Y} \rightarrow\mathbb{R}$ measures how the prediction of our predictive model conforms to the true target $Y$. 
For example, given a fitted response function $f(x)$, we can take the score to be the absolute residual $V(x,y) = |y-f(x)|$. 
In the split conformal prediction framework, the dataset is split into a training set for training the function $f$ and a calibration set to calculate the non-conformity scores. 
For the remaining of the paper, we will assume the non-conformity score is fixed and differentiable, which encompasses many popular choices of non-conformity scores such as $V(x,y) = |y-f(x)|$\cite{gibbs2023conformal}, $V(x,y) = |y-f(x)|/\sigma(x)$\cite{Angelopoulos2021-ur}, or $V(x,y) = \max\{\hat{q}_l(x)-y, y-\hat{q}_h(x)\}$ \cite{romano2019conformalized}. 

Given the non-conformity score function, the conformal prediction algorithm will output the regions of $y$ where the scores are small. The threshold $q^*$ for the nominal miscoverage rate $\alpha$ is the $\lfloor (n+1)(1-\alpha)\rfloor/n$-quantile of the conformity scores on the calibration set. The predictive set is formally defined as 
\begin{align}\label{eqn:predict_set}
    C(X_{n+1}) = \{y: V(X_{n+1},y)\leq q^*\}. 
\end{align} 

We can prove that $C(\cdot)$ satisfies the marginal coverage guarantee by exchangability:

\begin{theorem}[Marginal Coverage Guarantee \cite{vovk2005algorithmic}]\label{thm:mc}
    Assume $\{(X_i, Y_i)\}_{i=1}^n$ are independent and identically distributed, then the split conformal prediction set satisfies 
    \begin{align}
        P(Y_{n+1}\in C(X_{n+1})) \geq 1-\alpha. 
    \end{align}
    If $V(X_{n+1},Y_{n+1})$ has a continuous distribution, then 
    \begin{align}
        P(Y_{n+1}\in C(X_{n+1})) \leq 1-\alpha+\frac{1}{n+1}. 
    \end{align}
\end{theorem}

In practice, we are often more interested in the conditional coverage for any given $X=x$, i.e., $\mathbb{P}(Y \in \hat{C}_{\alpha}(X) | X=x) = 1-\alpha$. A key observation is that the threshold $q^*$ is taken as the quantile over the marginal distribution of the non-conformity scores. However, the conditional coverage $\mathbb{P}(Y \in \hat{C}_{\alpha}(X) | X=x) = \mathbb{P}(V(X,Y) \leq q^* | X=x)$ depends on the conditional distribution of the non-conformity scores. The discrepancy of the score distribution motivates the following proposition.

\begin{proposition}\label{prop:mc_and_cc}
If $P(V) = P(V|X=x)$, then $\mathbb{P}(Y \in \hat{C}_{\alpha}(X) | X=x) \geq 1-\alpha$.
If $V(X,Y)$ has a continuous distribution, then $P(Y\in C(X)|X=x) \leq 1-\alpha+\frac{1}{n+1}. $
\end{proposition}
\begin{proof}
    This is a direct result from \Cref{thm:mc}. 
\end{proof}


\Cref{prop:mc_and_cc} reveals that the main reason why conformal prediction cannot achieve good conditional coverage is $P(V) \neq P(V|X)$. It implies a sufficient condition for achieving perfect conditional coverage (for confidence level $\alpha$). 
By this observation, we can train the regression function by regularizing the distance between the marginal and conditional non-conformity score distribution. 
Assuming the machine learning model is trained by minimizing mean squared error (MSE), then we can optimize the function $f_\theta(\cdot)$ by 
\begin{align}
    \min_\theta \mathbb{E} (y-f_\theta(x))^2 + \lambda \sup_x d(P(V|X=x), P(V)).
    \label{eq:first}
\end{align}
The positive $\lambda$ is the hyperparameter that balances the mean squared loss and the distance constraint. Here we first write $d(\cdot, \cdot)$ as a general distance function. In what follows, we study what distance metric should be chosen. 

\subsection{Connection with Kolmogorov–Smirnov (KS) Distance}

We find a proper distance measure to achieve valid conditional coverage is the Kolmogorov–Smirnov (KS) distance. The connection is established based on the following proposition. 
\begin{proposition}[Conditional Coverage Rate]
\label{prop:ccr}
    Denote the cumulative distribution functions (CDFs) for the marginal and conditional non-conformity score distributions $P(V)$ and $P(V|X=x)$ are $F(v)$ and $G_x(v)$, respectively, then the asymptotic conditional coverage rate  for the split conformal prediction is $G_x(F^{-1}(1-\alpha))$. 
\end{proposition}

We include all the proofs in \Cref{sec:proof}. 
For a given nominal level $1-\alpha$, we then minimize the difference between the conditional coverage of the conformal prediction and the nominal level $|1-\alpha-G_x(F^{-1}(1-\alpha))|$ to achieve the desired coverage rate. However, this requires us to train a separate model for each $\alpha$, which may create unnecessary computation costs. Hence, we choose to minimize the objective over all possible $\alpha$, and the objective becomes 

\begin{align}
     \max_\alpha  |1-\alpha-G_x(F^{-1}(1-\alpha))| =& \max_v |F(v) - G_x(v)| \label{eq:worst} \\
     =& \text{KS}(P(V), P(V|X=x)) \label{eq:ksdef}
\end{align}
where the first equality comes from setting $v = F^{-1}(1-\alpha)$, and the second equality is by the definition of the KS distance. 




We can thus achieve the target conditional coverage by controlling the KS distance between the conditional and marginal non-conformity score distributions. Taking the KS distance as an additive regularization, the objective in \Cref{eq:first} becomes 

\begin{align}
    \min_\theta \mathbb{E} (y-f_\theta(x))^2 + \lambda \sup_x \text{KS}(P(V|X=x), P(V)).
\end{align}

Here we use KS distance as the distance function $d(\cdot,\cdot)$ in \Cref{eq:first}. The $\sup_x$ aims to improve the worst conditional coverage for all possible $x$. 

\subsection{Optimizing KS Distance}

In practice, we do not have access to the true distributions of either marginal or conditional non-conformity distributions, which makes estimating the KS distance challenging. We thus choose to use empirical CDFs to estimate them. 
If we have access to random samples $\{V_i\}_{i=1}^n$ and $\{V_i^x\}_{i=1}^n$ drawn from $P(V)$ and $P(V|X=x)$ respectively. The empirical KS distance can be written as 

\begin{align}\label{eqn:emp_ks}
     \max_t |\sum_{i=1}^n \frac{1}{n}\mathbb{I}[V_i \leq t] - \sum_{i=1}^n \frac{1}{n}\mathbb{I}[ V_i^x\leq t]| 
\end{align}

However, it cannot be optimized easily since the indicator function is not differentiable. To address the optimization problem, we use the sigmoid function with a temperature parameter to approximate the indicator function. More specifically, 

\begin{align}\label{eqn:ksapprox}
       \widehat{\text{KS}}(P(V|X=x), P(V)) = \max_t |\sum_{i=1}^n \frac{1}{n}\sigma(\gamma(t-V_i)) - \sum_{i=1}^n \frac{1}{n}\sigma(\gamma(t-V_i^x))|, 
\end{align}
where $\sigma(x) = \frac{1}{1+\exp(-x)}$ and $\gamma$ is the temperature. Since the score is unidimensional, $t$ can be selected from a grid. When $\gamma \rightarrow \infty$, \Cref{eqn:ksapprox} recovers \Cref{eqn:emp_ks}. 

To prevent potential overfitting, we choose to optimize MSE and KS distance on the training and calibration sets separately. Assume we have a training set $\{X_i,Y_i\}_{i=1}^{n_1}$, and a calibration set $\{X_i,Y_i\}_{i=n_1+1}^{n_1+n_2}$, we sample $n_s$ samples $\{V_i\}_{i \in \mathcal{I}}$, $\mathcal{I} \subset \{n_1+1, \cdots, n_1+n_2\}$, $|\mathcal{I}| = n_s$. 
While we can sample from the marginal non-conformity score distribution, we do not have access to the conditional distribution. Therefore, we fit a conditional generative model $P_\phi(y|x)$ on the training data and sample $\hat{Y}$ given $X$ to estimate the KS distance. Specifically, we sample $\hat{y}_i^j \sim P_\phi(y|x_i), j = 1, \cdots, n_s$, from the fitted conditional generative model and compute the score $V_i^{j,x} := V(x_i, \hat{y}_i^j)$.
Then the empirical objective can be written as 

\begin{align}
\label{eqn:emp_objective}
     & \frac{1}{n_1}\sum_{i=1}^{n_1} (y_i - f_\theta(x_i))^2 \nonumber \\
     & + \lambda \max_{n_1+1\leq i \leq n_1+n_2} \left\{\max_t |\sum_{i' \in \mathcal{I}} \frac{1}{n_s}\sigma(\gamma(t-V_{i'})) - \sum_{j=1}^{n_s} \frac{1}{n_s}\sigma(\gamma(t-V_i^{j,x}))| \right\}
\end{align}

The complete algorithm is included in \Cref{alg:main}. 

\begin{algorithm}
\caption{KS-constrained Conformal Prediction (KS-CP)}\label{alg:main}
\begin{algorithmic}
\Require $D_\text{train} = \{X_i,Y_i\}_{i=1}^{n_1}$, $D_\text{calib} = \{X_i,Y_i\}_{i=n_1+1}^{n_1 + n_2}$, $n_s$, function class $f_\theta(\cdot)$, non-conformity score function $V(\cdot,\cdot)$.  
\State Train function $f_\theta(X)$ on $D_\text{train}$ by minimizing MSE. 
\State Get $V$ on $D_\text{calib}$ using $f_\theta(X)$. 
    \State Train conditional density model $P_\phi(Y|X)$ on $D_\text{train}$. 
\For{n epoch}
\State Train the regression function by \Cref{eqn:emp_objective}.
\State Update $V$ using $f_\theta$.     
\EndFor 
\State Get conformal intervals $C_\alpha(X)$ using $D_\text{calib}$ by \Cref{eqn:predict_set}.
\end{algorithmic}
\end{algorithm}

Since our algorithm uses the same calibration step as the split conformal prediction framework, we have the same marginal coverage guarantee in \Cref{thm:mc}. Since we used approximated conditional non-conformity distribution, we are interested in whether our algorithm can still optimize for the conditional miscoverage rate. Next we show that by using a fitted conditional distribution, we will minimize the upper bound of the coverage discrepancy from the nominal level. Based on the  triangle inequality of  KS distance \citep{johnston2019berry}, we have

\begin{align}
    \text{KS}(P(V|X),P(V)) \leq \underbrace{\text{KS}(P(V|X),P_\phi(V|X))}_{\text{Generative Model Error on } V} + \underbrace{\text{KS}(P_\phi(V|X),P(V))}_{\text{Our Regularization}}.
\end{align}

The first term is the estimation error of the conditional generative model. Since we use the conditional generative model to fit $y$, ideally we want the upper bound to be dependent on the estimation error of $y$ instead of $V$. 
\Cref{prop:dpi} connects the KS distance of the score distribution and the KS distance of the outcome distribution. For some non-conformity scores, we can then write 

\begin{align}
    \text{KS}(P(V|X),P(V)) \leq \underbrace{2\text{KS}(P(Y|X),P_\phi(Y|X))}_{\text{Generative Model Error on } Y} + \underbrace{\text{KS}(P_\phi(V|X),P(V))}_{\text{Our Regularization}}
\end{align}

\begin{proposition}\label{prop:dpi}
    If $V(x,Y) = |Y-f(x)|$, $V(x,Y) = |Y-f(x)|/\sigma(x)$, or $V(x,Y) = \max\{\hat{f}_{{\alpha_{\text{lo}}}}-Y, Y-\hat{f}_{{\alpha_{\text{hi}}}}\}$, assume the distribution $P_Y,Q_Y,P_V,Q_V$ have CDFs $F_Y,G_Y,F_V,G_V$, respectively, then $\text{KS}(P_V(V),Q_V(V)) \leq 2\text{KS}(P_Y(Y),Q_Y(Y))$. 
\end{proposition}

\Cref{prop:dpi} justifies why we only need to fit the conditional generative model on $y$ once instead of repeatedly train it on $V$ every time $f_\theta$ is updated, which greatly reduces the computation time.  The following proposition connects the optimization objective and the asymptotic conditional coverage. 

\begin{proposition}\label{prop:asy}
Under the conditions of \Cref{prop:ccr,prop:dpi}, suppose the conditional density estimator $p(y|x)$ is consistent and the regularization $\text{KS}(P_\phi(V|X),P(V))$ is reduced to be below $\epsilon$ by optimization, then  $1-\alpha - \epsilon \leq \mathbb{P}(Y \in \hat{C}_{\alpha}(X) | X=x) \leq 1-\alpha + \epsilon$ as $n \to \infty$. 
\end{proposition}

This implies that asymptotically, if our regularization term converges to a small value by optimization, the deviation from the nominal level will also be small, therefore our algorithm can effectively improve the conditional coverage for the predictive set output by conformal prediction. In \Cref{sec:proof}, we show the existence of consistent estimator $\hat{p}(y|x)$ by nonparametric constructions. Note that the precision $\epsilon$ depends on the regularization strength $\lambda$. A large $\lambda$ improves the conditional coverage close to the nominal level but may affect the predictive set size due to the potential error of the predictive model. 
We offer a detailed discussion of the impact of $\lambda$ and $\gamma$ in \Cref{app:discussion}. We discuss the computation cost of our algorithm compared to traditional conformal prediction algorithm in \Cref{app:computation}.  
In what follows, we will empirically show the benefit of the proposed method through a comprehensive synthetic data example and real-world datasets. 

\section{Experiments}

We use both synthetic and real-world datasets to validate the proposed method. The code is available at \url{https://github.com/ruijiang81/adjusting_reg_model}. 

\noindent \textbf{Nonconformity Score:} We use three popular forms of conformity scores in our experiments:
\begin{itemize}
    \item Residual: $V(x,Y) = |Y-f(x)|$. It quantifies the absolute error between the true target and the predicted value. 
    \item Normalized: $V(x,Y) = |Y-f(x)|/\sigma(x)$, where $\sigma(x)$ is the estimated standard deviation output by a generated model or a Bayesian neural network. Compared to the residual score, the normalized score is adaptive and can form confidence intervals that are wider for data points that have higher uncertainty. We use the conditional generative models to calculate the standard deviation. 
    \item Quantile: $V(x,Y) = \max\{\hat{f}_{{\alpha_{\text{lo}}}}-Y, Y-\hat{f}_{{\alpha_{\text{hi}}}}\}$, where $\hat{f}_{{\alpha_{\text{lo}}}}, \hat{f}_{{\alpha_{\text{hi}}}}$ are the predicted quantiles output by method like quantile regression. For confidence level $\alpha$, we choose $\alpha_\text{lo} = (1-\alpha)/2$ and $\alpha_\text{hi} = (1+\alpha)/2$. 
\end{itemize}

\noindent \textbf{Baseline:} We compare against the following baselines:

\begin{itemize}
    \item Conformal Prediction (CP): We use the standard split conformal prediction method for each conformity score. 
    \item Orthogonal Quantile Regression \cite{feldman2021improving} (OQR): OQR improves standard quantile regression by regularizing on an objective, which is a necessary condition after perfect conditional coverage is achieved. We also compare against the corresponding conformalized algorithm Conformal Orthogonal Quantile Regression (COQR). 
    \item Generative Model: Since our algorithm uses a generative model to improve conditional coverage, a natural choice is to use the generative model directly for uncertainty quantification. As we shall see, the generative model generally does not enjoy the marginal coverage guarantee and may have worse conditional coverage. 
\end{itemize}

See \Cref{app:baselines} for a detailed discussion about the baselines. 

\begin{figure}
\centering
\begin{subfigure}{.47\textwidth}
  \centering
  \includegraphics[width=\linewidth]{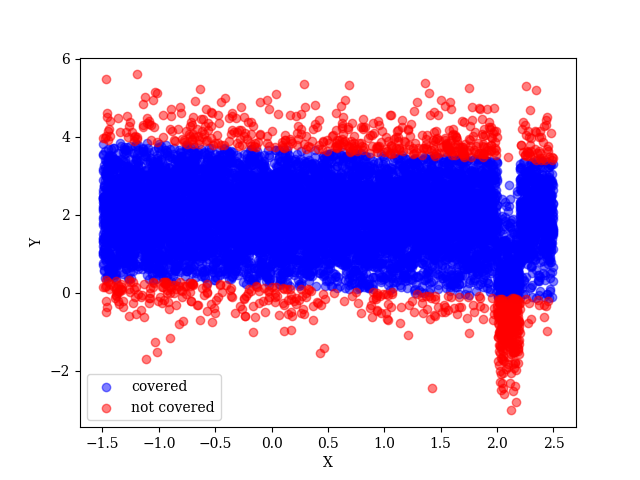}
  \caption{CP: 53.7\%}
  \label{fig:sub1}
\end{subfigure}%
\begin{subfigure}{.47\textwidth}
  \centering
  \includegraphics[width=\linewidth]{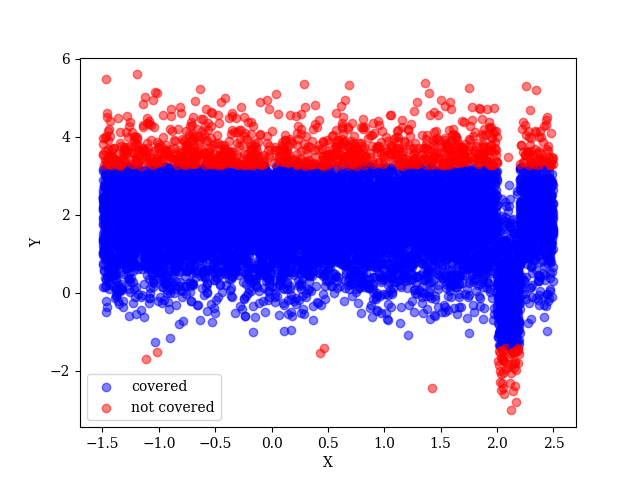}
  \caption{KS-CP: 87.0\%}
  \label{fig:sub2}
\end{subfigure} \\
\begin{subfigure}{.47\textwidth}
  \centering
  \includegraphics[width=\linewidth]{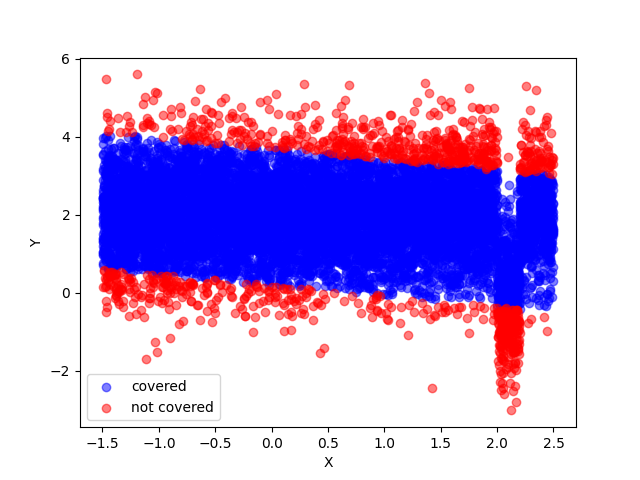}
  \caption{OQR: 56.8\%}
  \label{fig:sub3}    
\end{subfigure}
\begin{subfigure}{.47\textwidth}
  \centering
  \includegraphics[width=\linewidth]{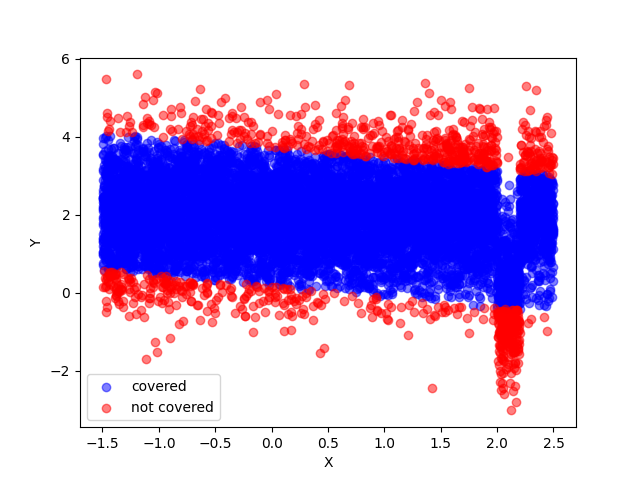}
  \caption{COQR: 65.0\%}
  \label{fig:sub4} 
\end{subfigure}\\
\begin{subfigure}{.47\textwidth}
  \centering
  \includegraphics[width=\linewidth]{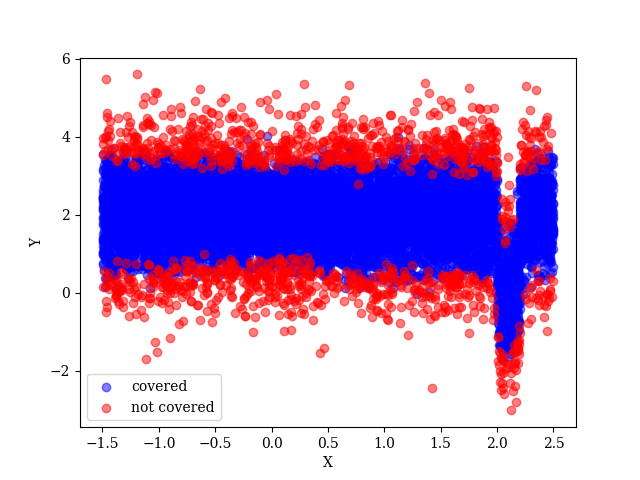}
  \caption{MDN: 56.0\%}
  \label{fig:sub5} 
\end{subfigure}
\begin{subfigure}{.47\textwidth}
  \centering
  \includegraphics[width=\linewidth]{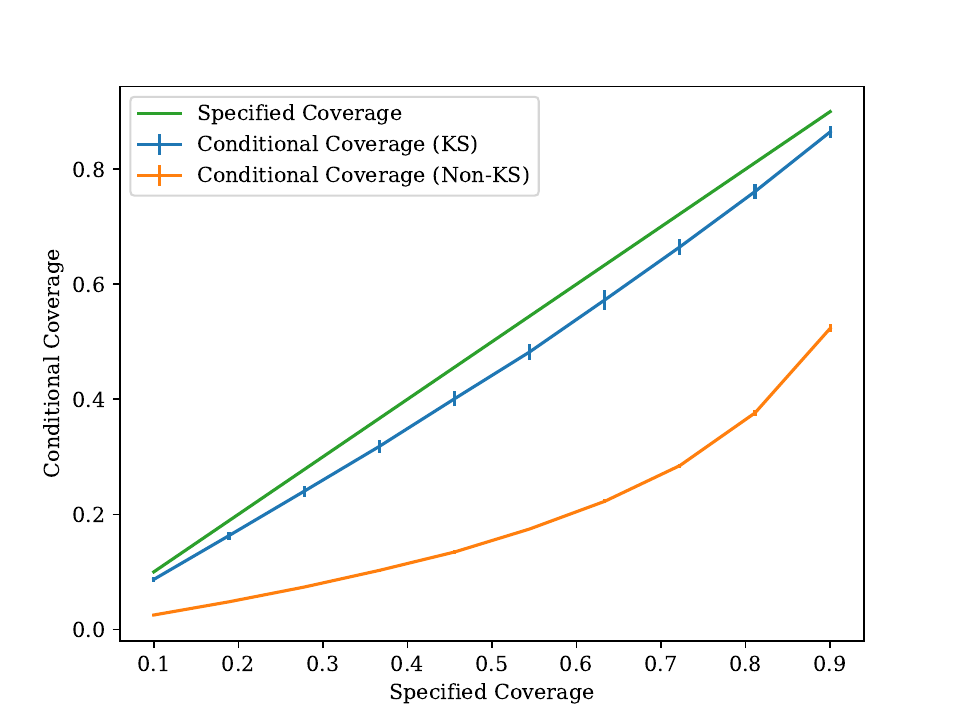}
  \caption{Conditional Coverage for all $\alpha$}
  \label{fig:syn1_allalpha}
\end{subfigure}
\caption{Coverage under Synthetic Data (Setting I) with Linear Regression, $1-\alpha=90\%$. Here we show the conditional coverage for each method. Our method can achieve the specified conditional coverage while all other methods have significantly lower conditional coverage.}
\label{fig:syn1_vis}
\end{figure}

\subsection{Synthetic Data}
We use two synthetic data-generating processes to illustrate the benefit of our method. We use linear regression as the model class for the synthetic data and report the performance trained with mean squared loss (MSE) and our algorithm. We report the marginal coverage, conditional coverage, set size, and MSE in \Cref{tab:syndata} for $\alpha=0.9$. We use the Mixture Density Network \cite{bishop1994mixture} as the conditional generative model. We choose $\lambda=1000$ and $\gamma=10$. We use 2000 training samples, 1000 calibration samples, and 10000 test samples. We repeat the experiment over five runs. 

For the first setting, we consider the setting where some subgroups may be undercovered. $X\sim U[-1.5, 2.5]$, $Y = \begin{cases}
      0 + \epsilon& \text{if } 2 \leq X\leq 2.2\\
      2 + \epsilon& \text{otherwise}
    \end{cases}  $, where $\epsilon\sim\mathcal{N}(0,1)$, then a linear regression model would significantly undercover for the subgroup $2 \leq X\leq 2.2$. This corresponds to a case where some small subgroup of the population has a different target distribution. 

The visualization of the first setup is included in \Cref{fig:syn1_vis}. Since the subgroup has a small population, a simple linear regression algorithm trained with MSE loss will ignore the subgroup and provide uninformative confidence intervals for this group. Therefore, the standard conformal prediction has poor conditional coverage as shown in \Cref{fig:sub1}. By constraining on the worst conditional miscoverage measured by the KS distance, KS-CP intentionally outputs a function with worse MSE to achieve a near-perfect conditional coverage. We can also see while existing methods like OQR and COQR improve the conditional coverage, it still has a significantly worse conditional coverage for the small subpopulation compared to our method. While our method relies on a conditional generative model, a conditional generative model generally does not have the finite-sample marginal coverage guarantee as demonstrated in \Cref{tab:syndata}. Even for such a simple toy example, the marginal coverage of the MDN model does not equal the marginal coverage while all conformal prediction methods enjoy the distributionally-free guarantee. Also, we find that MDN has a significantly lower conditional coverage than the nominal coverage, which is likely because of the smoothness property of neural networks \cite{grinsztajn2022tree}. 
We also show the conditional coverage for our method in \Cref{fig:syn1_allalpha} across $\alpha$ for the residual non-conformity score. Compared to the standard conformal prediction algorithm, our algorithm significantly improves the conditional coverage across the specified nominal coverages $\alpha$. 

In the visualization, we can see that KS-CP increases the set size globally for all $x$. We would like to note this effect is evitable for some non-conformity scores like residual and normalized scores since a different $f$ can only impact $q^*$, which globally affect the set size (the set size of residual nonconformity score is $2q^*$). However, for other nonconformity scores such as the quantile score, the change in the set size can be adaptive to $x$, which we give an example in \Cref{app:adaptive}. 

For the second data-generating process, we assume $X\sim U[-1.5, 2.5]$, $Y = \epsilon$, where $\epsilon\sim\mathcal{N}(0,1)$. The second synthetic data setting serves as a sanity check when the vanilla conformal prediction method can achieve perfect coverage. The visualization and quantitative results are shown in \Cref{fig:syn2_vis} in Appendix and \Cref{tab:syndata}, respectively. In this case, almost all methods achieve the perfect conditional coverage since the condition in \Cref{prop:mc_and_cc} is directly satisfied by the MSE minimizing function. In this case, our method also outputs the correct function without a significantly larger cost in the set size and MSE.

\begin{table}[!ht]
    \centering
    \begin{tabular}{ccccccc}
        Nonconformity score & Dataset & Method & MC & CC & Set Size & MSE \\ \toprule
        Residual & Syn I & CP & 0.90(0.00) & 0.55(0.01) & 3.62(0.04) & 1.17(0.01) \\
        Residual & Syn I & KS-CP & 0.89(0.01) & \textbf{0.87}(0.01) & 4.59(0.10) & 2.11(0.09) \\ \hline
        Normalized & Syn I & CP & 0.91(0.00) & 0.44(0.03) & 3.64(0.03) & 1.17(0.01) \\ 
        Normalized & Syn I & KS-CP & 0.90(0.00) & \textbf{0.65}(0.04) & 4.16(0.15) & 1.59(0.07) \\ \hline
        Quantile & Syn I & CP & 0.91(0.00) & 0.69(0.01) & 3.63(0.03) & NA \\ 
        Quantile & Syn I & KS-CP & 0.90(0.01) & \textbf{0.87}(0.01) & 4.18(0.09) & NA \\ 
        Quantile & Syn I & COQR & 0.90(0.00) & 0.53(0.04) & 3.65(0.05) & NA \\ \hline
        NA & Syn I & CDE-MDN & 0.87(0.00) & 0.63(0.04) & 3.10(0.03) & 1.03(0.01) \\ 
        NA & Syn I & OQR & 0.65(0.05) & 0.24(0.05) & 2.05(0.22) & NA \\ \hline 
        Residual & Syn II & CP & 0.91(0.00) & \textbf{0.90}(0.01) & 3.34(0.03) & 1.00(0.01) \\ 
        Residual & Syn II & KS-CP & 0.90(0.01) & 0.89(0.01) & 3.93(0.23) & 1.41(0.14) \\ \hline
        Normalized & Syn II & CP & 0.89(0.00) & 0.74(0.01) & 3.22(0.05) & 1.00(0.01) \\ 
        Normalized & Syn II & KS-CP & 0.91(0.00) & \textbf{0.76}(0.02) & 3.88(0.28) & 1.34(0.20) \\ \hline
        Quantile & Syn II & CP & 0.91(0.00) & 0.90(0.01) & 4.03(0.03) & NA \\ 
        Quantile & Syn II & KS-CP & 0.91(0.00) & 0.94(0.01) & 4.06(0.07) & NA \\ 
        Quantile & Syn II & COQR & 0.90(0.00) & \textbf{0.95}(0.00) & 4.02(0.05) & NA \\ \hline
        NA & Syn II & CDE-MDN & 0.87(0.00) & 0.68(0.01) & 3.10(0.03) & 1.01(0.01) \\ 
        NA & Syn II & OQR & 0.89(0.00) & 0.94(0.00) & 3.89(0.04) & NA \\ \bottomrule
    \end{tabular}
    \caption{Quantitative results for each method with synthetic data. $1-\alpha=90\%$. We report Marginal Coverage (MC), Conditional Coverage (CC), Set Size and Mean Squared Error (MSE). KS-CP consistently improves conditional coverage over standard conformal prediction across settings. The best method with the same non-conformity score is made bold.}
    \label{tab:syndata}
\end{table}

\begin{table}[!htbp]
\resizebox{\textwidth}{!}{%
\begin{tabularx}{\textwidth}{XXlXXXX}
NC Score   & Dataset & Method   & MC         & WSLAB      & Set Size   & MSE        \\ \toprule
Residual   & bike        & CP       & 0.90(0.00) & 0.73(0.00) & 0.76(0.00) & 0.07(0.00) \\
Residual   & bike        & KS-CP    & 0.90(0.00) & \textbf{0.79}(0.01) & 2.49(0.01) & 0.60(0.00) \\ \hline 
Normalized & bike        & CP       & 0.90(0.00) & 0.75(0.01) & 0.76(0.00) & 0.07(0.00) \\
Normalized & bike        & KS-CP    & 0.90(0.00) & \textbf{0.79}(0.00) & 2.36(0.01) & 0.56(0.00) \\ \hline 
Quantile   & bike        & CP       & 0.90(0.00) & \textbf{0.89}(0.02) & 0.68(0.01) & NA         \\
Quantile   & bike        & KS-CP    & 0.91(0.00) & \textbf{0.89}(0.02) & 2.18(0.02) & NA         \\
Quantile   & bike        & COQR     & 0.89(0.00) & 0.83(0.03) & 2.10(0.03) & NA         \\ \hline 
Non-CP     & bike        & CDE-CVAE & 0.84(0.00) & 0.69(0.00) & 1.86(0.01) & 0.50(0.00) \\
Non-CP     & bike        & OQR      & 0.89(0.02) & 0.81(0.02) & 2.08(0.08) & NA         \\ \hline 
Residual   & communities & CP       & 0.90(0.00) & 0.80(0.01) & 3.39(0.06) & 1.09(0.03) \\
Residual   & communities & KS-CP    & 0.89(0.00) & \textbf{0.82}(0.06) & 3.42(0.15) & 1.20(0.18) \\ \hline 
Normalized & communities & CP       & 0.89(0.00) & \textbf{0.81}(0.02) & 3.02(0.05) & 1.09(0.03) \\
Normalized & communities & KS-CP    & 0.91(0.00) & \textbf{0.81}(0.02) & 4.15(0.21) & 1.80(0.16) \\ \hline 
Quantile   & communities & CP       & 0.90(0.01) & 0.81(0.03) & 2.79(0.08) & NA         \\
Quantile   & communities & KS-CP    & 0.89(0.01) & \textbf{0.84}(0.04) & 1.79(0.09) & NA         \\
Quantile   & communities & COQR     & 0.91(0.01) & \textbf{0.84}(0.02) & 2.20(0.09) & NA         \\ \hline 
Non-CP     & communities & CDE-CVAE & 0.74(0.00) & 0.50(0.00) & 1.20(0.01) & 0.39(0.00) \\
Non-CP     & communities & OQR      & 0.29(0.14) & 0.25(0.13) & 0.45(0.23) & NA         \\ \hline 
Residual   & parkinsons  & CP       & 0.90(0.00) & 0.75(0.01) & 0.40(0.01) & 0.03(0.00) \\
Residual   & parkinsons  & KS-CP    & 0.91(0.00) & \textbf{0.83}(0.01) & 2.89(0.03) & 0.74(0.02) \\ \hline 
Normalized & parkinsons  & CP       & 0.89(0.00) & 0.75(0.01) & 0.41(0.03) & 0.03(0.00) \\
Normalized & parkinsons  & KS-CP    & 0.91(0.00) & \textbf{0.79}(0.00) & 3.01(0.09) & 0.81(0.04) \\ \hline 
Quantile   & parkinsons  & CP       & 0.90(0.00) & \textbf{0.84}(0.02) & 0.35(0.04) & NA         \\
Quantile   & parkinsons  & KS-CP    & 0.90(0.00) & \textbf{0.84}(0.02) & 3.03(0.04) & NA         \\
Quantile   & parkinsons  & COQR     & 0.90(0.00) & 0.82(0.01) & 2.58(0.08) & NA         \\ \hline 
Non-CP     & parkinsons  & CDE-CVAE & 0.89(0.00) & 0.76(0.01) & 2.63(0.01) & 0.72(0.00) \\
Non-CP     & parkinsons  & OQR      & 0.88(0.01) & 0.80(0.01) & 2.46(0.08) & NA \\ \hline 
Residual   & meps19  & CP       & 0.90(0.00) & 0.74(0.01) & 3.27(0.01) & 1.01(0.01) \\ 
Residual   & meps19  & KS-CP    & 0.91(0.00) & \textbf{0.77}(0.01) & 2.81(0.01) & 0.76(0.00) \\ \hline 
Normalized & meps19  & CP       & 0.91(0.00) & 0.67(0.01) & 3.19(0.01) & 1.01(0.01) \\
Normalized & meps19  & KS-CP    & 0.91(0.00) & \textbf{0.81}(0.00) & 2.81(0.01) & 1.00(0.03) \\ \hline 
Quantile   & meps19  & CP       & 0.90(0.00) & 0.68(0.03) & 2.68(0.04) & NA         \\
Quantile   & meps19  & KS-CP    & 0.91(0.00) & \textbf{0.83}(0.03) & 5.30(0.76) & NA         \\
Quantile   & meps19  & COQR     & 0.91(0.00) & 0.69(0.03) & 2.71(0.04) & NA         \\  \hline 
Non-CP     & meps19  & CDE-CVAE & 0.85(0.00) & 0.75(0.01) & 2.01(0.00) & 0.60(0.00) \\
Non-CP     & meps19  & OQR      & 0.82(0.02) & 0.62(0.05) & 2.15(0.11) & NA         \\ \hline 
Residual   & meps20  & CP       & 0.91(0.00) & 0.75(0.00) & 3.08(0.01) & 0.86(0.00) \\
Residual   & meps20  & KS-CP    & 0.91(0.00) & \textbf{0.81}(0.00) & 2.84(0.01) & 0.72(0.00) \\ \hline 
Normalized & meps20  & CP       & 0.91(0.00) & 0.73(0.01) & 2.98(0.01) & 0.86(0.00) \\
Normalized & meps20  & KS-CP    & 0.91(0.00) & \textbf{0.82}(0.00) & 2.75(0.01) & 0.75(0.01) \\ \hline 
Quantile   & meps20  & CP       & 0.91(0.00) & 0.82(0.02) & 2.70(0.04) & NA         \\
Quantile   & meps20  & KS-CP    & 0.90(0.00) & \textbf{0.86}(0.03) & 3.92(0.03) & NA         \\
Quantile   & meps20  & COQR     & 0.91(0.00) & 0.78(0.01) & 2.67(0.03) & NA         \\ \hline 
Non-CP     & meps20  & CDE-CVAE & 0.82(0.00) & 0.74(0.00) & 2.15(0.00) & 0.67(0.00) \\
Non-CP     & meps20  & OQR      & 0.84(0.01) & 0.69(0.03) & 2.25(0.07) & NA         \\ \hline 
Residual   & meps21  & CP       & 0.91(0.00) & 0.73(0.01) & 3.05(0.01) & 0.87(0.00) \\
Residual   & meps21  & KS-CP    & 0.91(0.00) & \textbf{0.80(0.00)} & 2.75(0.01) & 0.70(0.00) \\ \hline 
Normalized & meps21  & CP       & 0.90(0.00) & 0.79(0.01) & 2.94(0.01) & 0.86(0.01) \\
Normalized & meps21  & KS-CP    & 0.90(0.00) & \textbf{0.84(0.01)} & 2.68(0.02) & 0.73(0.01) \\ \hline 
Quantile   & meps21  & CP       & 0.91(0.00) & 0.77(0.01) & 2.73(0.02) & NA         \\
Quantile   & meps21  & KS-CP    & 0.90(0.00) & \textbf{0.84(0.01)} & 6.36(0.48) & NA         \\
Quantile   & meps21  & COQR     & 0.90(0.00) & 0.76(0.01) & 2.72(0.03) & 0.00(0.00) \\ \hline 
Non-CP     & meps21  & CDE-CVAE & 0.84(0.00) & 0.82(0.00) & 2.11(0.01) & 0.62(0.00) \\
Non-CP     & meps21  & OQR      & 0.82(0.02) & 0.70(0.02) & 2.26(0.05) & NA\\ 
\bottomrule
\end{tabularx}
}
\caption{Quantitative metrics of each method with UCI datasets. $1-\alpha=90\%$. We report Marginal Coverage (MC), WSLAB, Set Size and Mean Squared Error (MSE). KS-CP improves conditional coverage over other conformal prediction methods across datasets. The best method with the same non-conformity score is made bold.}
\label{tab:ucidata}
\end{table}

\begin{figure}[!htbp]
\centering
\begin{subfigure}{.45\textwidth}
  \centering
  \includegraphics[width=\linewidth]{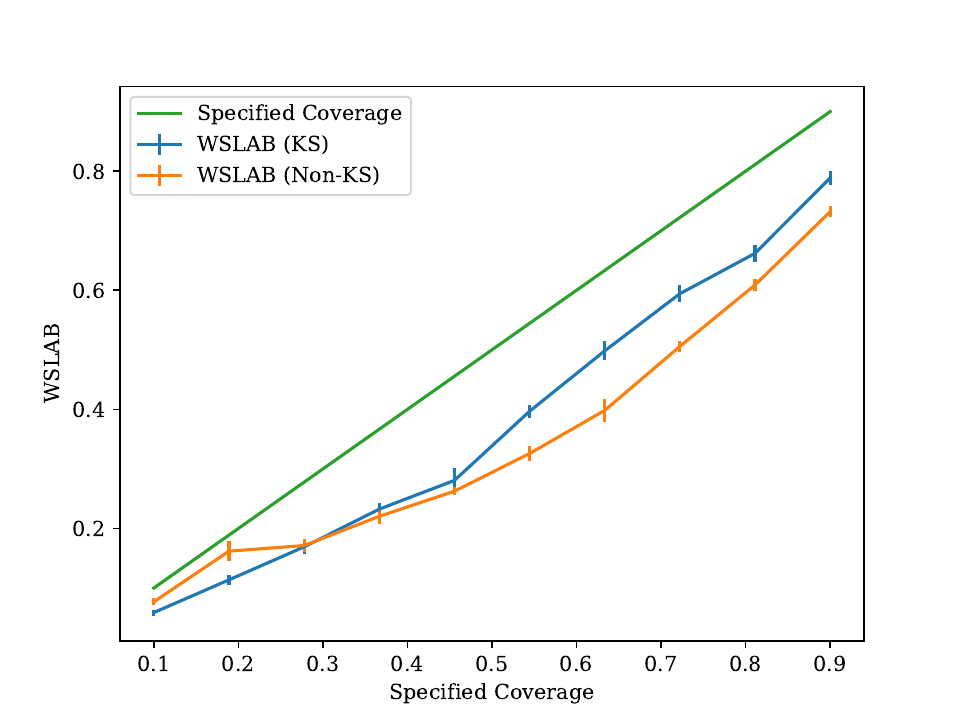}
  \caption{Bike}
  \label{fig:bike}
\end{subfigure} 
\begin{subfigure}{.45\textwidth}
  \centering
  \includegraphics[width=\linewidth]{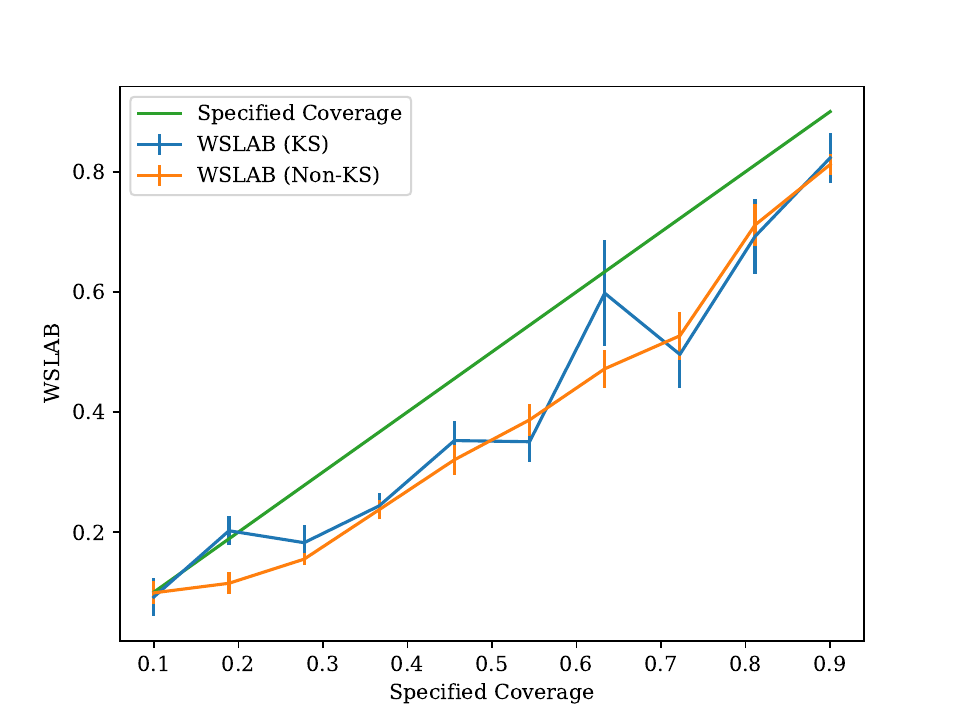}
  \caption{Communities}
  \label{fig:communities}
\end{subfigure}\\
\begin{subfigure}{.45\textwidth}
  \centering
  \includegraphics[width=\linewidth]{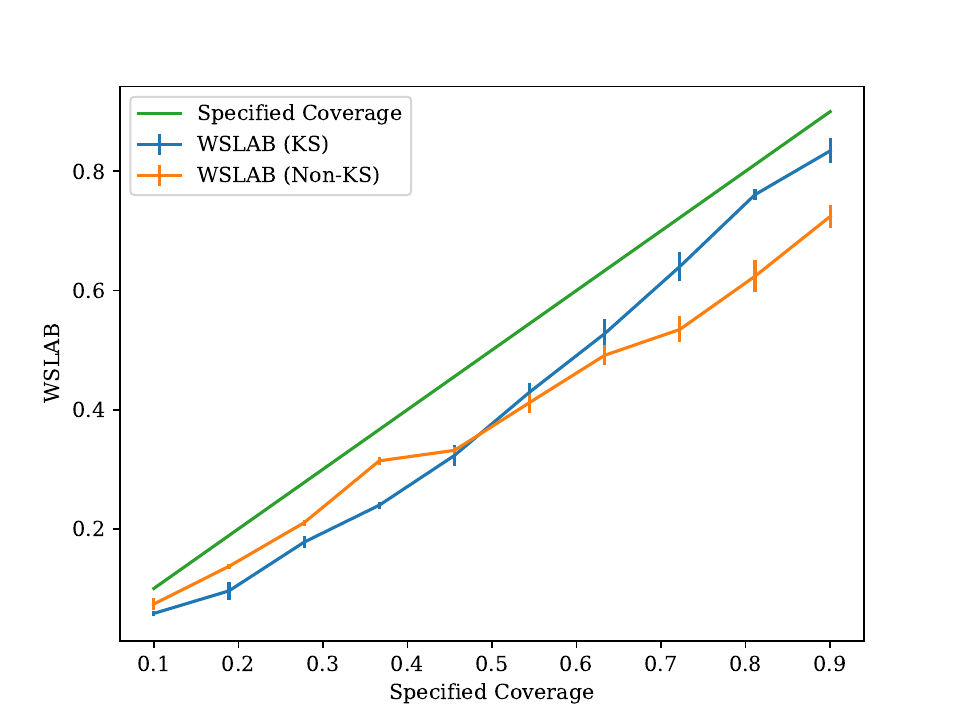}
  \caption{Parkinsons}
  \label{fig:parkinsons}
\end{subfigure}
\begin{subfigure}{.45\textwidth}
  \centering
  \includegraphics[width=\linewidth]{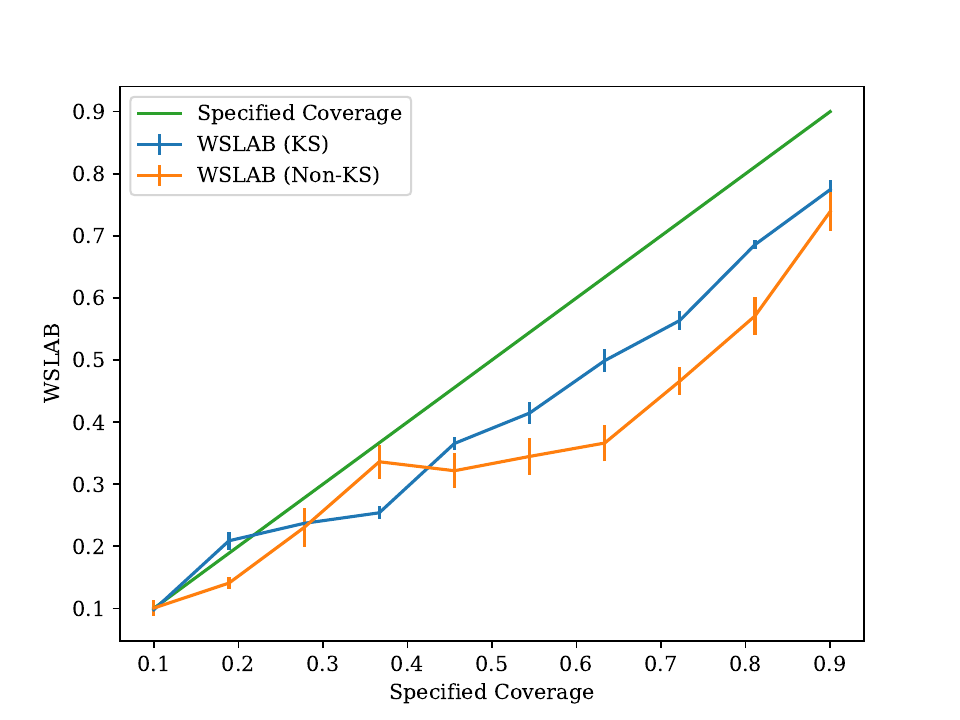}
  \caption{MEPS-19}
  \label{fig:meps19}
\end{subfigure} \\ 
\begin{subfigure}{.45\textwidth}
  \centering
  \includegraphics[width=\linewidth]{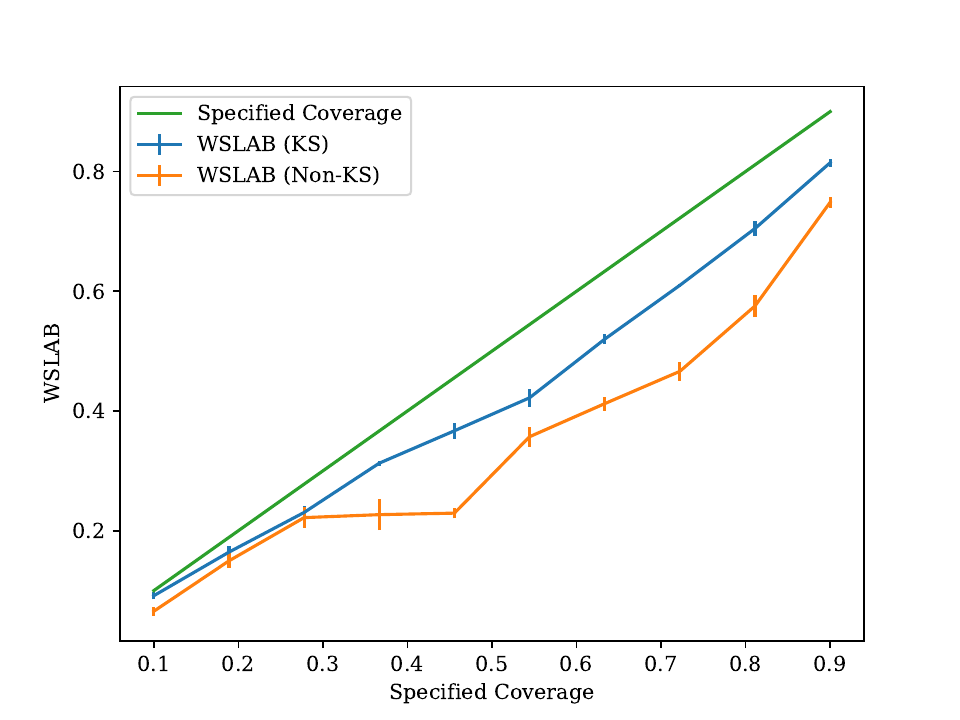}
  \caption{MEPS-20}
  \label{fig:meps20}
\end{subfigure} 
\begin{subfigure}{.45\textwidth}
  \centering
  \includegraphics[width=\linewidth]{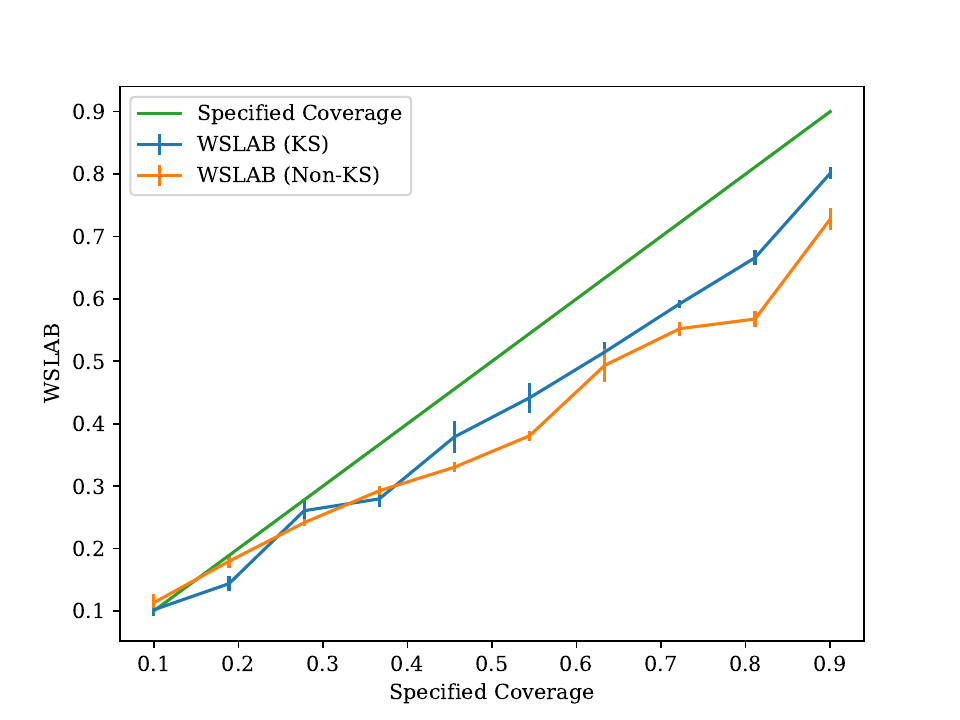}
  \caption{MEPS-21}
  \label{fig:meps21}
\end{subfigure}%
\caption{WSLAB for UCI Datasets with residual score across $\alpha$. Our method consistently improves WSLAB among all $\alpha$.}
\label{fig:ucidata}
\end{figure}

\subsection{Real-World Data}


We use 6 UCI datasets to further validate the proposed method. The complete data statistics is included in \Cref{app:datastat}. The model class is a three-layer feed-forward neural network with LeakyReLU activation. We choose $\lambda=100$ and $\gamma=10$. 
Since we do not have access to the underlying data-generating process using real-world datasets, we use the worst-slab coverage (WSLAB) \cite{Cauchois2021KnowingWY} as the approximated conditional coverage. WSLAB measures the worst coverage on randomly sampled slabs on the test data, which is a surrogate measure of the worst group coverage. 
If the WSLAB has a good coverage, the method should have a good subgroup coverage over all slabs. 
Moreover, it is efficient to compute with a $O(n)$ time \citep{chung2005optimal} and has been used in many conformal prediction papers as a surrogate measure to examine conditional coverage
\citep{Cauchois2021KnowingWY,feldman2021improving,wang2022probabilistic}.
 We use the conditional Variational AutoEncoder \cite{sohn2015learning} as the conditional generative model. We split the data into 50\% training, 20\% calibration, and 30\% test. We repeat the experiment over five runs.

We report the marginal coverage, conditional coverage, set size, and MSE in \Cref{tab:ucidata} for $\alpha=0.9$ for different methods and non-conformity scores. Our method improves WSLAB across different choices of non-conformity scores and datasets. For each non-conformity score, our method consistently improves over the standard conformal prediction method. We do not observe a consistent improvement in WSLAB for OQR and COQR. Similar to the synthetic data set, using the conditional generative model alone for uncertainty quantification does not have the finite-sample marginal coverage guarantee and empirically may have worse WSLAB than our method. 
We also report WSLAB across $\alpha$ in \Cref{fig:ucidata} for the residual non-conformity score. Compared to standard conformal prediction, our method consistently improves the worst WSLAB miscoverage among all $\alpha$.

\begin{figure}[!htbp]
\centering
\begin{subfigure}{.45\textwidth}
  \centering
  \includegraphics[width=\linewidth]{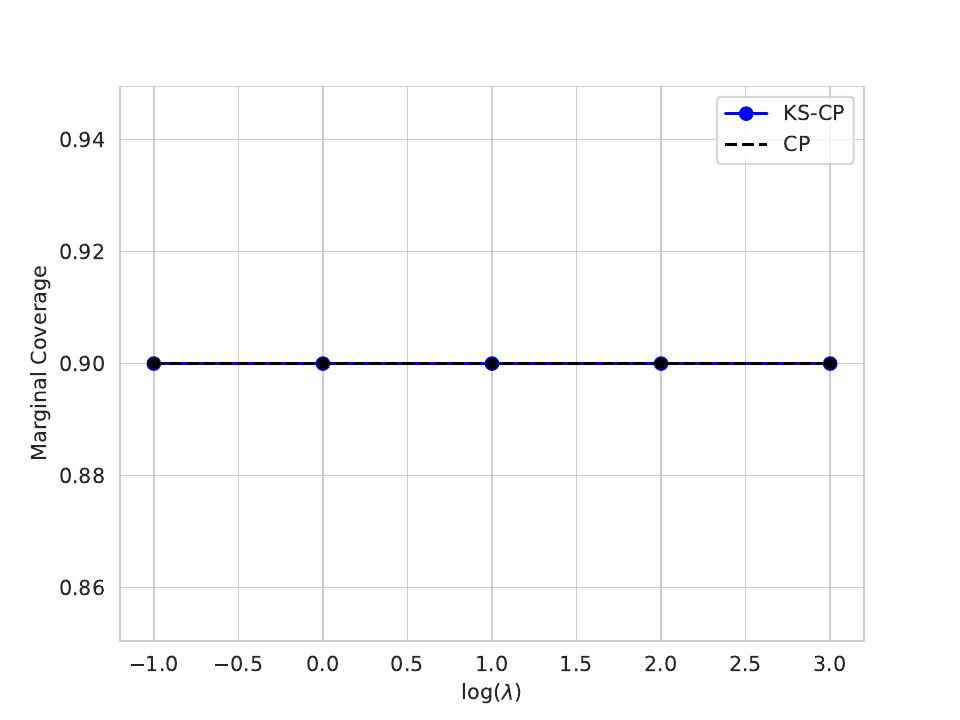}
  \caption{MC}
  \label{fig:ab_mc}
\end{subfigure}%
\begin{subfigure}{.45\textwidth}
  \centering
  \includegraphics[width=\linewidth]{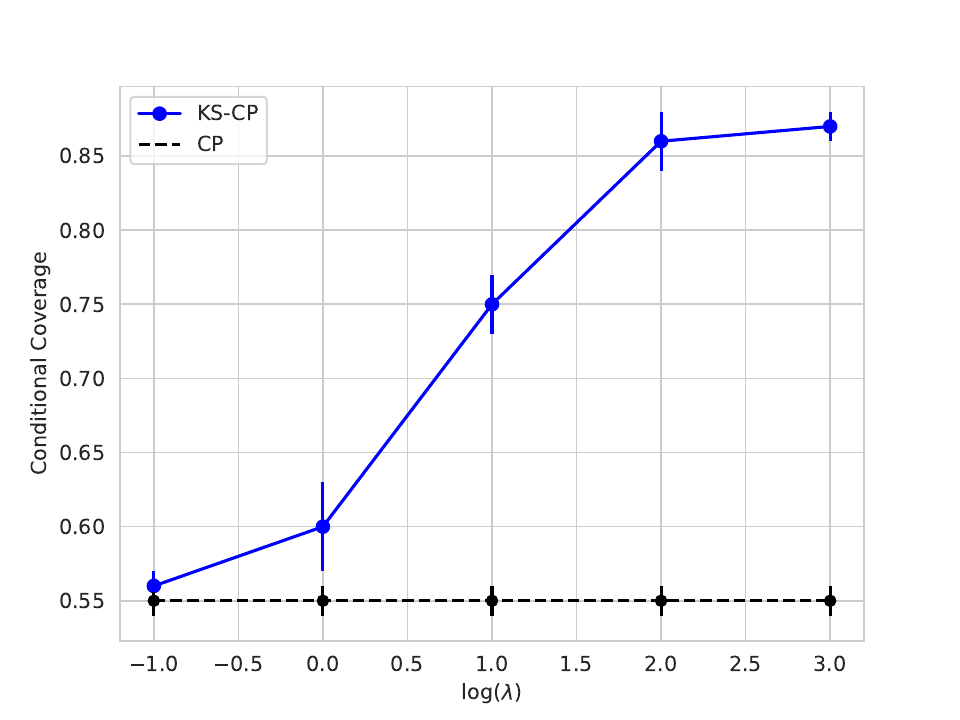}
  \caption{CC}
  \label{fig:ab_cc}
\end{subfigure} \\
\begin{subfigure}{.45\textwidth}
  \centering
  \includegraphics[width=\linewidth]{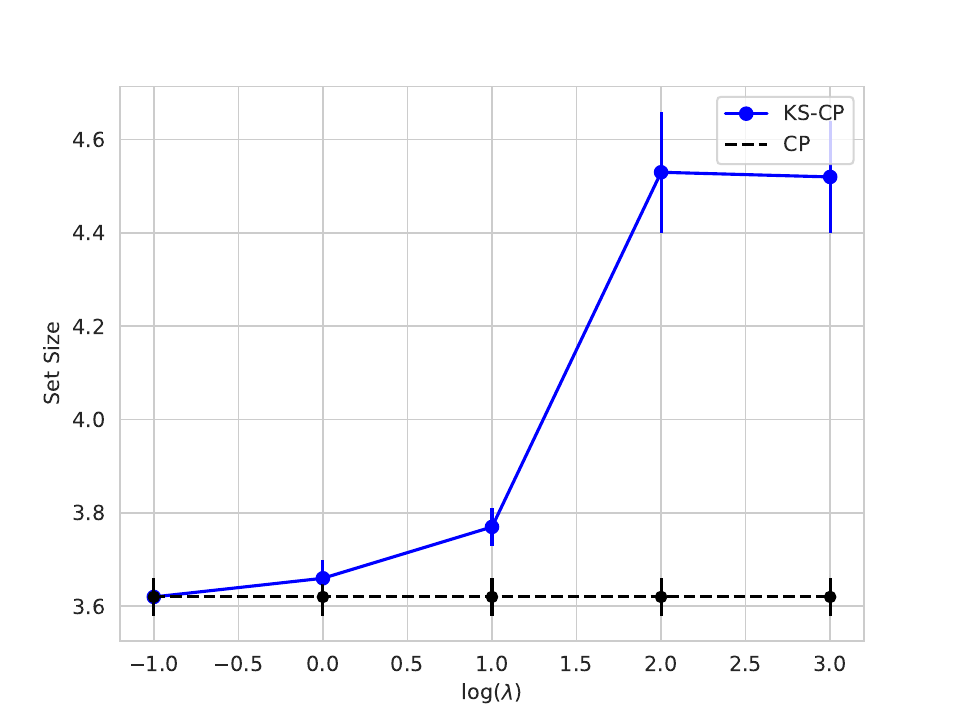}
  \caption{Set Size}
  \label{fig:ab_ss}
\end{subfigure}%
\begin{subfigure}{.45\textwidth}
  \centering
  \includegraphics[width=\linewidth]{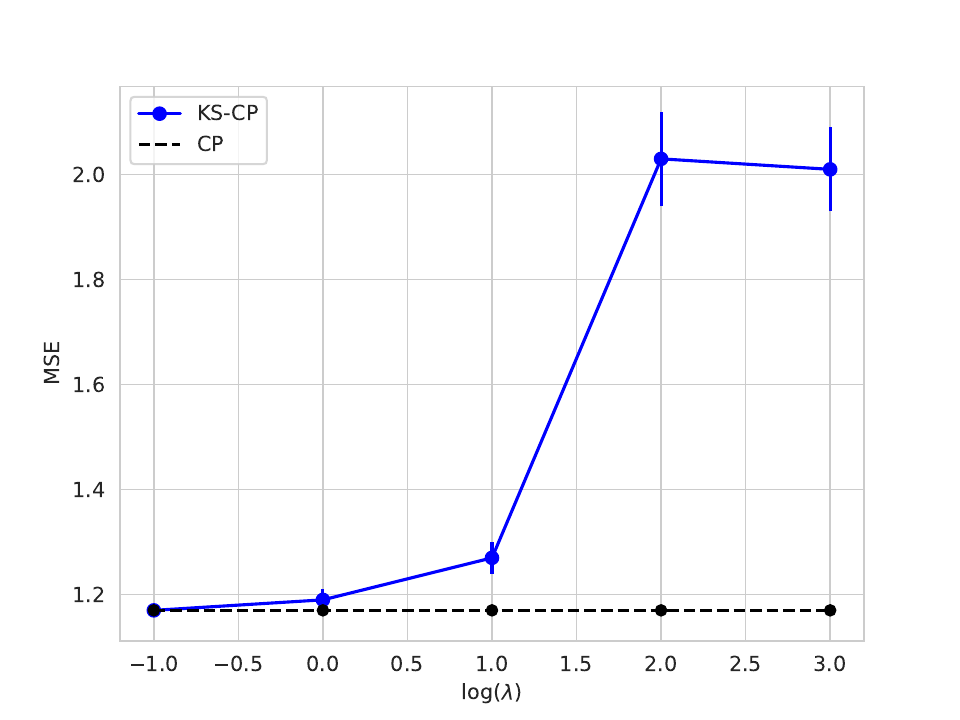}
  \caption{MSE}
  \label{fig:ab_mse}
\end{subfigure} 
\caption{Ablation Studies for Different Choices of $\lambda$ for Synthetic Data Setup I. We report the Marginal Coverage (MC), Conditional Coverage (CC), Set Size, and MSE for $\log(\lambda) = -1,0,1,2,3$ when $1-\alpha=90\%$.}
\label{fig:ablation}
\end{figure}

\subsection{Ablation Study}
In our objective \Cref{eqn:emp_objective}, we use $\lambda$ to balance the MSE loss and the conditional coverage. 
In this section, we show the results for our synthetic data settings for different values of $\lambda$. We report the Marginal Coverage, Conditional Coverage, Set Size and MSE for $\log(\lambda) = -1,0,1,2,3$ when $1-\alpha=90\%$ in \Cref{fig:ablation}. 

Since conformal prediction enjoys the marginal coverage guarantee, the marginal coverage is not affected by $\lambda$. As $\lambda$ increases, the confidence intervals of KS-CP get wider, and the conditional coverage gets closer to the true specified nominal coverage. As a result of this, the MSE of the fitted function also gets worse to get better conditional coverage for the worst subgroup. This suggests our regularization can effectively improve the conditional coverage for uncertainty quantification purposes and may sometimes be a competing objective to the standard MSE loss. We also report the ablation studies for synthetic data setting II in \Cref{app:ablation}. Since the standard conformal prediction method can achieve perfect conditional coverage in the second synthetic data setting, the conditional coverage is not affected by $\lambda$. Similarly, the set size and the MSE of the fitted function are also not affected much by $\lambda$, which is desired since we do not want to sacrifice the MSE for conditional coverage when the standard conformal prediction method can achieve perfect conditional coverage.
In practice, we can set $\lambda$ to a large number such as 100 or use the validation dataset and a surrogate measure such as WSLAB to select the best $\lambda$.

\section{Conclusion}

In this paper, we investigate how to train uncertainty-aware regression functions to improve the conditional coverage of the function after applying the conformal prediction procedure by matching the marginal and conditional non-conformity score distribution. We theoretically show that the worst conditional distribution miscoverage is upper bounded by the Kolmogorov–Smirnov distance between the marginal and conditional non-conformity score distribution and propose a novel algorithm to optimize it by leveraging a conditional generative model. We demonstrate the efficacy of the proposed method using synthetic and real-world datasets. However, our method still requires the conditional generative model to have enough capacity to have an informative upper bound; future work can consider how to derive a tighter upper bound for more efficient conditional coverage improvement. It would also be interesting to consider how to extend the proposed framework in a causal inference setup or non-stationary environment where the exchangeability assumption is violated. 

\FloatBarrier

\newpage
\backmatter













\newpage
\appendix
\begin{appendices}

\section{Proof}\label{sec:proof}

\begin{proof}[Proof for \Cref{prop:ccr}]
By \Cref{thm:mc}, assuming continuous nonconformity score, we have
    \begin{align}
        1-\alpha \leq P(V(X,Y)\leq q^*) \leq 1-\alpha +\frac{1}{n+1}.
    \end{align}
    Take the inverse CDF of $p(V)$, 
    \begin{align}
        F^{-1}(1-\alpha) \leq q^* \leq F^{-1}(1-\alpha+\frac{1}{n+1}).
    \end{align}
    Then take the CDF of $p(V|X)$, 
    \begin{align}
       G_x(F^{-1}(1-\alpha)) \leq P(V(X,Y)\leq q^*|X=x) = G_x(q^*) \leq G_x(F^{-1}(1-\alpha+\frac{1}{n+1})).
    \end{align}
    Thus the asymptotic coverage rate given $x$ is 
    \begin{align}
        G_x(F^{-1}(1-\alpha)),
    \end{align}
    when $n \to \infty$
\end{proof}

\begin{proof}[Proof for \Cref{prop:dpi}]
When $V(x,Y) = |Y-f(x)|$, 
\begin{align}
    &\text{KS}(P_V(V),Q_V(V)) = \max_v |F_V(v) - G_V(v)| \notag\\
    & = \max_v |P_V(V\leq v) - Q_V(V\leq v)| \notag\\ 
    & = \max_v |P_Y(|Y-f(x)|\leq v) - Q_Y(|Y-f(x)|\leq v)| \notag\\ 
    & = \max_v |P_Y( f(x) - v\leq Y \leq f(x) + v) - Q_Y( f(x) - v\leq Y \leq f(x) + v)|  \notag\\
    & = \max_v |P_Y(Y \leq f(x) + v) - P_Y(Y \leq f(x) - v) - Q_Y( Y \leq f(x) + v) + Q_Y( Y \leq f(x) - v)|  \notag\\
    & \leq \max_v |P_Y(Y\leq f(x)+v) - Q_Y(Y\leq f(x)+v)| \nonumber \\ & + \max_v |P_Y(Y\leq f(x)-v) - Q_Y(Y\leq f(x)-v)| \notag \\
    & = 2\text{KS}(P_Y(Y),Q_Y(Y)), 
\end{align}
where the inequality is by the triangle inequality. 

Similarly, for the normalized non-conformity score when $V(x,Y) = |Y-f(x)|/\sigma(x)$, 
\begin{align}
    &\text{KS}(P_V(V),Q_V(V)) = \max_v |F_V(v) - G_V(v)| \notag\\
    & = \max_v |P_Y(|Y-f(x)|\leq v\sigma(x)) - Q_Y(|Y-f(x)|\leq v\sigma(x))| \notag\\ 
    & = \max_v |P_Y( f(x) - v\sigma(x)\leq Y \leq f(x) + v\sigma(x)) - Q_Y( f(x) - v\sigma(x)\leq Y \leq f(x) + v\sigma(x))|  \notag\\
    & = \max_v |P_Y(Y \leq f(x) + v\sigma(x)) - P_Y(Y \leq f(x) - v\sigma(x)) \\
     & \qquad\qquad - Q_Y( Y \leq f(x) + v\sigma(x)) + Q_Y( Y \leq f(x) - v\sigma(x))|  \notag\\
    & \leq \max_v |P_Y(Y\leq f(x)+v\sigma(x)) - Q_Y(Y\leq f(x)+v\sigma(x))| \nonumber \\ & \qquad\qquad + \max_v |P_Y(Y\leq f(x)-v\sigma(x)) - Q_Y(Y\leq f(x)-v\sigma(x))| \notag \\
    & = 2\text{KS}(P_Y(Y),Q_Y(Y)).
\end{align}

For the quantile non-conformity score when $V(x,Y) = \max\{\hat{f}_{{\alpha_{\text{lo}}}}-Y, Y-\hat{f}_{{\alpha_{\text{hi}}}}\}$, 
\begin{align}
    &\text{KS}(P_V(V),Q_V(V)) = \max_v |F_V(v) - G_V(v)| \notag\\
    & = \max_v |P_Y(Y \leq \hat{f}_{{\alpha_{\text{hi}}}} + v) - P_Y(Y \leq \hat{f}_{{\alpha_{\text{lo}}}} - v) \\
     & \qquad\qquad - Q_Y( Y \leq \hat{f}_{{\alpha_{\text{hi}}}} + v) + Q_Y( Y \leq \hat{f}_{{\alpha_{\text{lo}}}} - v)|  \notag\\
    & \leq \max_v |P_Y(Y\leq \hat{f}_{{\alpha_{\text{hi}}}}+v) - Q_Y(Y\leq \hat{f}_{{\alpha_{\text{hi}}}}+v)| \nonumber \\ & \qquad\qquad + \max_v |P_Y(Y\leq \hat{f}_{{\alpha_{\text{lo}}}}-v) - Q_Y(Y\leq \hat{f}_{{\alpha_{\text{lo}}}}-v)| \notag \\
    & = 2\text{KS}(P_Y(Y),Q_Y(Y)).
\end{align}
\end{proof}

\begin{proof}[Proof for \Cref{prop:asy}]
Suppose the consistency of the conditional density estimator, i.e. $\lim_{n\to \infty}\text{KS}(p(y|x), \hat{p}(y | x)) =0
$. By \Cref{prop:ccr,prop:dpi} and the equation above, we have 
\begin{align}
&\lim_{n\to \infty} \max_\alpha |p(Y \in \hat{C}(X) | X) - (1-\alpha) | \notag \\
=&\max_\alpha |G_x(F^{-1}(1-\alpha)) - (1-\alpha) | \notag \\
=& \text{KS}(p(v|x),p(v)) \notag \\
\leq & 2 \lim_{n\to \infty} \text{KS}(p(y|x),p_\phi(y|x)) + \text{KS}(p_\phi(v|x),p(v)) \notag \\
= & \epsilon
\end{align}
Hence, for $\forall~\alpha$, the deviation of conditional coverage from $1-\alpha$ is less than $\epsilon$ as $n \to \infty$.

Finally, we show the Nadaraya-Watson conditional density estimator (CDE) \citep{nakayama2011asymptotic}  and the least-squares CDE \citep{sugiyama2010least} as 
 examples of consistent $\hat{p}(y | x)$\footnote{We refer to \citet{bilodeau2023minimax} for the consistency results of a family of conditional density estimators bounded by the empirical Hellinger entropy.}. 
Consider a kernel density estimation of $p(x) = \frac1n\sum_i K_h(||x-x_i||_2)$. The Nadaraya-Watson estimator is constructed as a density ratio 
\begin{align}
\hat{p}(y | x) = \frac{\hat{p}(y, x)}{\hat{p}(x)}  = \frac{\sum_i K_{h_2}(||x-x_i||_2)K_{h_1}(y-y_i)}{\sum_i K_{h_2}(||x-x_i||_2)}. 
\end{align}
Asymptotic approximations show that if $h_1 \to 0$, $h_2 \to 0$, and $nh_1h_2 \to \infty$, then $\hat{p}(y | x) \to {p}(y | x)$ for $\forall$ $y$ as $n \to \infty$ \citep{holmes2007fast}. For example, a typical choice of $h_1, h_2 \sim \mathcal{O}(n^{-1/3})$ gives $|\hat{p}(y | x) - {p}(y | x)| \sim \mathcal{O}(n^{-1/3})$. Therefore, we have
\begin{align}
\lim_{n\to \infty}\text{KS}(p(y|x), \hat{p}(y | x)) =& \lim_{n\to \infty}\max_y |F(y|x) - \hat{F}(y | x)| \notag \\
\leq &   \max_y \int_{-\infty}^y \lim_{n\to \infty}|p(y'|x) - \hat{p}(y' | x)| dy' \notag \\
=& 0
\end{align}

The least-square CDE is defined as 
\begin{align*}
    \hat{p}(y|x) = \argmin_{\hat{r}} \frac{1}{2} \iint (\hat{r}(x,y) - r(x,y))^2p(x)dxdy
\end{align*}
The consistency of $\hat{r}$ is given in Theorem 1 of \citet{sugiyama2010least}. 
\end{proof}

\section{Baseline}
\label{app:baselines}

Here we describe each baseline in detail. T

\noindent \textbf{Conformal Prediction (CP):} The Conformal Prediction (CP) is the typical split conformal prediction. It randomly splits data into a training set and a validation set. A regression function $\hat{y} = f_\theta(x)$ is fit on the training data, and the fitted function is used to compute the nonconformity score on the validation set. The complete algorithm is shown in \Cref{alg:cp}.

\noindent \textbf{Orthogonal Quantile Regression (OQR):} OQR is based on the observation that a necessary condition for conditional coverage is the independence between the coverage identifier $V=\mathbf{1}[Y \in \hat{C}(X)]$ and the set size $L=\hat{C}(X)$. It proposes an orthogonal regularization based on Pearson correlation as 
$R(V,L) = \frac{\text{Cov}(V,L)}{\sqrt{\text{Var}(V)\text{Var}(L)}}$,
or using Hilbert-Schmidt independence criterion (HSIC) to account for nonlinear dependence. The objective function of OQR is an additive combination of quantile regression loss and the orthogonal regularization. 

\noindent \textbf{Generative model (MDN):} We use mixture density network (MDN) for our toy example. MDN models the data likelihood as 
\begin{align}
p(y|x) = \sum_{i=1}^M \pi_i \phi(y;x, \theta_i).
\end{align}
$\pi_i$ are the mixture component weights, and $\phi(y;x, \theta) = \mathcal{N}(\mu_{\theta_i}(x), \text{diag}(\sigma^2_{\theta_i}(x)))$.

\noindent \textbf{Generative model (CVAE):} We use Conditional Variational AutoEncoder (CVAE) in the real-world data experiments. For condition $c$, latent $z$, and target $x$, CVAE optimizes the variational lower bound 

\begin{align}
   \mathbb{E}[\log p(x|c,z)]  - \text{KL}(q(z|x,c)\|p(z|c)). 
\end{align}

\begin{algorithm}
\caption{Conformal Prediction}\label{alg:cp}
\begin{algorithmic}
\Require $D_\text{train} = \{X_i,Y_i\}_{i=1}^{n_1}$, $D_\text{calib} = \{X_i,Y_i\}_{i=n_1+1}^{n_1 + n_2}$, $n_s$, function class $f_\theta(\cdot)$, non-conformity score function $V(\cdot,\cdot)$.  
\State Train function $f_\theta(X)$ on $D_\text{train}$ by minimizing MSE. 
\State Get non-conformity score $V$ on $D_\text{calib}$ using $f_\theta(X)$. 
\State Get conformal intervals $C_\alpha(X)$ using $D_\text{calib}$ by \Cref{eqn:predict_set}.
\end{algorithmic}
\end{algorithm}

\section{Data Statistics}
\label{app:datastat}

The data statistics of UCI datasets is included in \Cref{tab:ucidatastat}. 

\noindent \textbf{Bike:} The bike sharing dataset contains hourly and daily count of rental bikes between years 2011 and 2012 in Capital bikeshare system with the corresponding weather and seasonal information. We use the daily count of rental bikes as the predictive target.

\noindent \textbf{Communities:} The communities and crime dataset  contains the demographic feature of a community and its corresponding per capita crime rate.

\noindent \textbf{Parkinsons:} The Parkinsons dataset contains biomedical voice measurements from 42 people with early-stage Parkinson's disease recruited to a six-month trial of a telemonitoring device for remote symptom progression monitoring. The dataset contains 20 biomedical voice measurements from each patient.

\noindent \textbf{MEPS19, MEPS20, MEPS21:} The MEPS datasets contain the medical expenditure panel survey data from 2019 to 2021. The datasets contains features like demographics, marriage status, health status, and health insurance status.  
The predictive target is the total medical expenditure.

\begin{table}[h]
\begin{tabular}{lcc}
Dataset     & Number of Samples & Number of Features \\ \toprule 
Bike        & 10886                      & 18                          \\ \hline 
Communities & 1994                       & 99                          \\\hline 
Parkinsons  & 5875                       & 20                          \\\hline 
MEPS\_19    & 10428                      & 139                         \\\hline 
MEPS\_20    & 17541                      & 139                         \\\hline 
MEPS\_21    & 15656                      & 139                        \\\bottomrule 
\end{tabular}
\caption{Data Statistics of UCI Datasets}
\label{tab:ucidatastat}
\end{table}

\section{Additional Results on Synthetic Data}
\label{app:add}
\Cref{fig:syn2_vis} shows the coverage results for the  Setting II of synthetic data. Most methods can achieve nearly perfect conditional coverage in this setting except non-conformal methods like MDN and OQR.  Like most methods, KS-CP can achieve near-perfect conditional coverage in this case while CP can achieve the perfect conditional coverage in this case by design. 

Compared to traditional conformal prediction method, KS-based method needs to estimate the conditional density model. When the conditional density model is not perfect (as in Fig. E2e), the regularization term of KS-based method is an upper bound for the true conditional coverage. We also like to note while the generative model has a relatively poor conditional coverage, KS-method still achieves near-perfect conditional coverage in the case.

\section{Additional Results on Ablation Studies}
\label{app:ablation}

We show the results for ablation studies for different choices of $\lambda$ for synthetic data setup II in \Cref{fig:ablation2}.
Since in this setting, the conditional coverage for each method is already close to the true specified nominal coverage, the conditional coverage is not affected much as $\lambda$ increases. 
Similarly, the set size and MSE are also not affected much as $\lambda$ increases, which is desired in this setting. 

\begin{figure}[!htbp]
\centering
\begin{subfigure}{.45\textwidth}
  \centering
  \includegraphics[width=\linewidth]{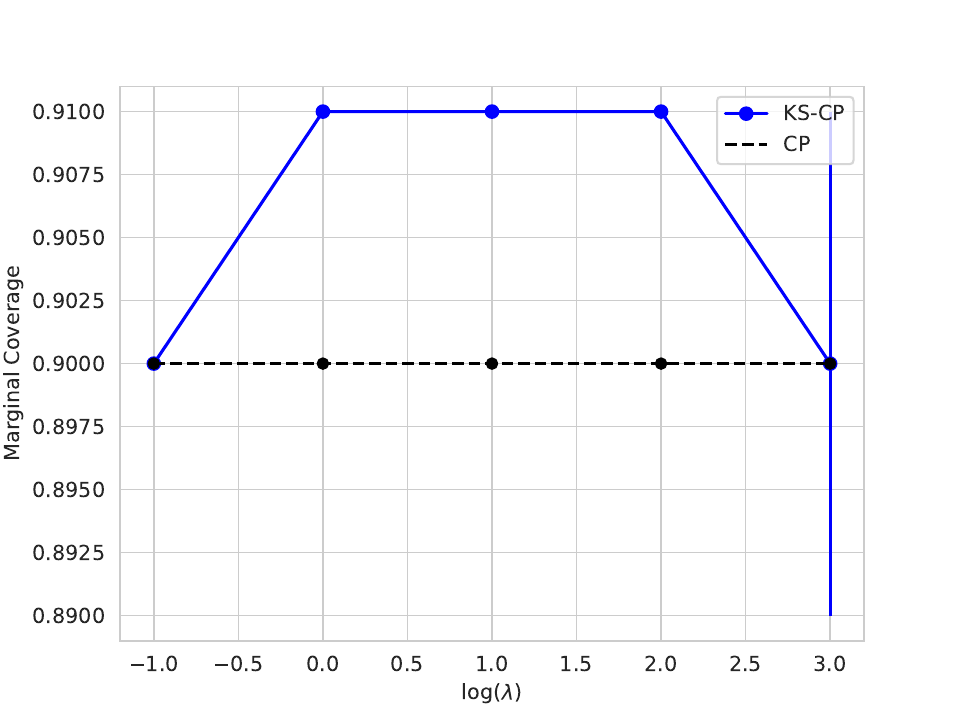}
  \caption{MC}
  \label{fig:ab_mc2}
\end{subfigure}%
\begin{subfigure}{.45\textwidth}
  \centering
  \includegraphics[width=\linewidth]{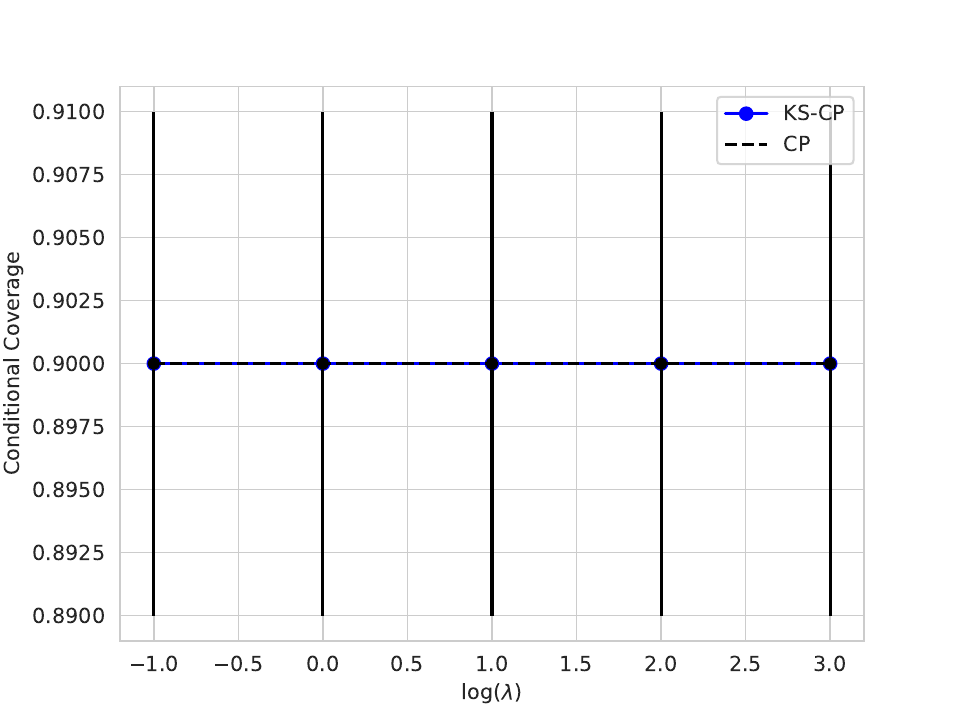}
  \caption{CC}
  \label{fig:ab_cc2}
\end{subfigure} \\
\begin{subfigure}{.45\textwidth}
  \centering
  \includegraphics[width=\linewidth]{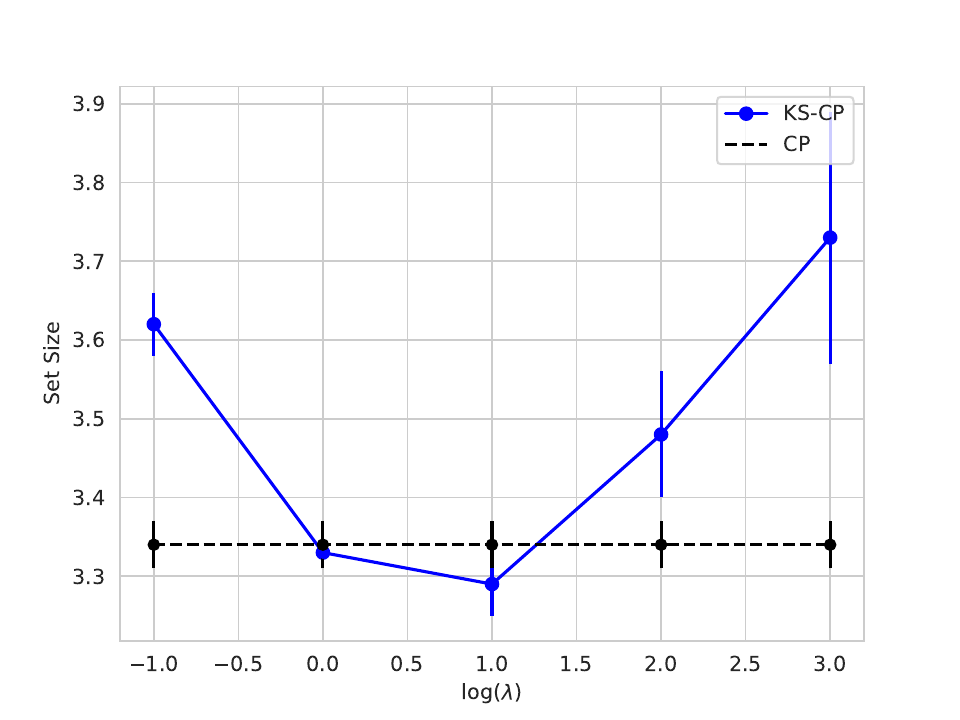}
  \caption{Set Size}
  \label{fig:ab_ss2}
\end{subfigure}%
\begin{subfigure}{.45\textwidth}
  \centering
  \includegraphics[width=\linewidth]{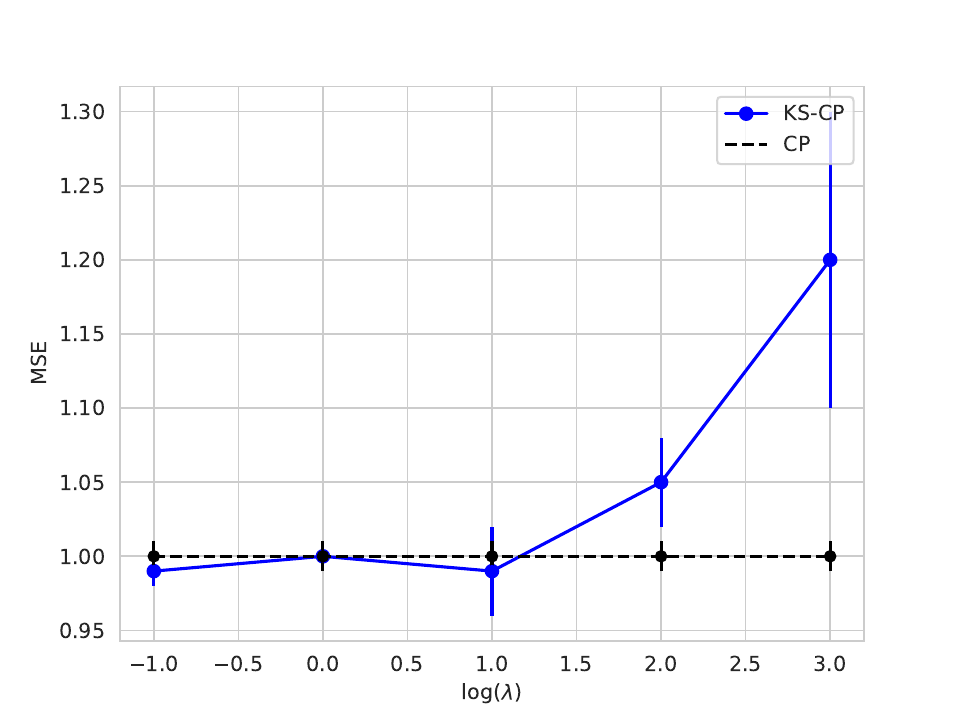}
  \caption{MSE}
  \label{fig:ab_mse2}
\end{subfigure}
\caption{Ablation Studies for Different Choices of $\lambda$ for Synthetic Data Setup II. We report the Marginal Coverage (MC), Conditional Coverage (CC), Set Size, and MSE for $\log(\lambda) = -1,0,1,2,3$ when $1-\alpha=90\%$.}
\label{fig:ablation2}
\end{figure}

\begin{figure}
\centering
\begin{subfigure}{.47\textwidth}
  \centering
  \includegraphics[width=\linewidth]{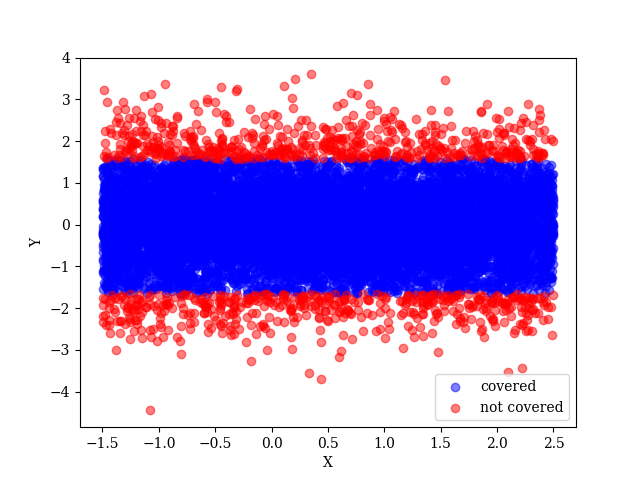}
  \caption{CP: 89.0\%}
  \label{fig:sub21}
\end{subfigure}%
\begin{subfigure}{.47\textwidth}
  \centering
  \includegraphics[width=\linewidth]{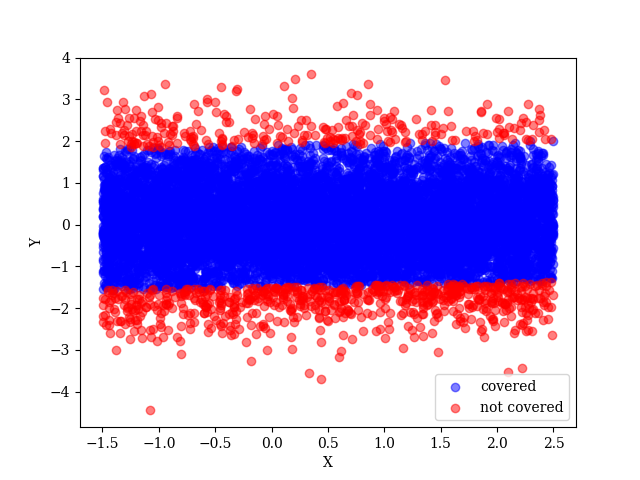}
  \caption{KS-CP: 89.0\%}
  \label{fig:sub22}
\end{subfigure} \\
\begin{subfigure}{.47\textwidth}
  \centering
  \includegraphics[width=\linewidth]{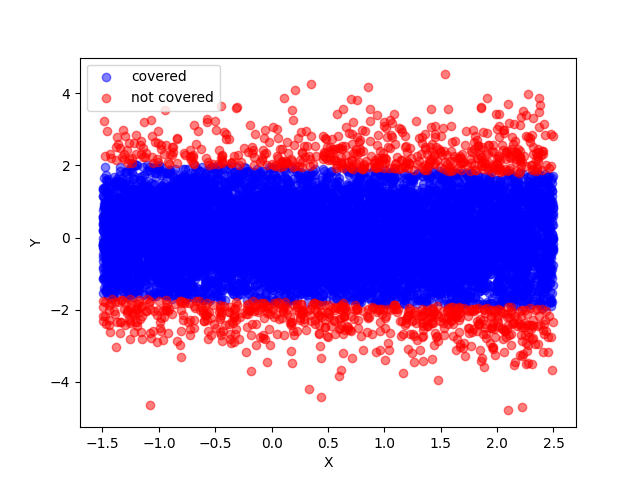}
  \caption{OQR: 93.0\%}
  \label{fig:sub23}    
\end{subfigure}
\begin{subfigure}{.47\textwidth}
  \centering
  \includegraphics[width=\linewidth]{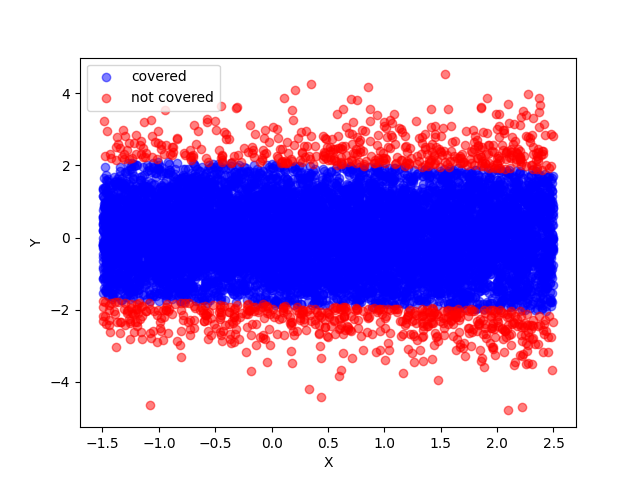}
  \caption{COQR: 94.0\%}
  \label{fig:sub24} 
\end{subfigure}\\
\begin{subfigure}{.47\textwidth}
  \centering
  \includegraphics[width=\linewidth]{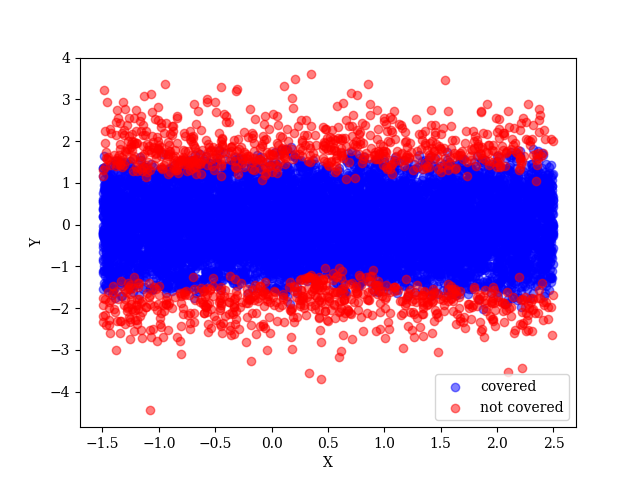}
  \caption{MDN: 68.0\%}
  \label{fig:sub25} 
\end{subfigure}
\begin{subfigure}{.47\textwidth}
  \centering
  \includegraphics[width=\linewidth]{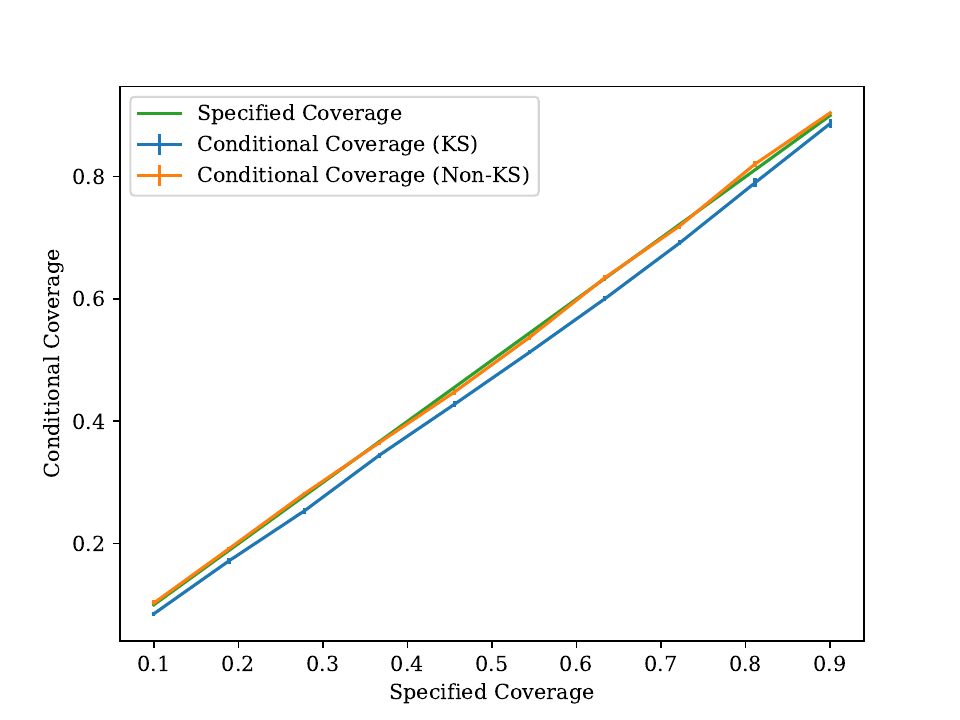}
  \caption{Conditional Coverage for all $\alpha$}
  \label{fig:sub26}
\end{subfigure}
\caption{Coverage under Synthetic Data (Setting II) with Linear Regression, $1-\alpha=90\%$. Here we show the conditional coverage for each method. Like most methods, KS-CP can achieve perfect conditional coverage in this case.}
\label{fig:syn2_vis} 
\end{figure}

\section{Discussion on Hyperparameter Selection}
\label{app:discussion}

\subsection{Selection of $\lambda$}
The parameter is used to control the tradeoff between prediction accuracy and the conditional coverage. The selection of $\lambda$ is highly application-dependent and relies on the requester's need on the robustness of the confidence interval. For example, in the application of drug discovery, a good conditional coverage may be preferred over prediction accuracy, then the requester can set a relatively large $\lambda$. In practice, the requester can set a maximum tolerance level of the predictive accuracy (e.g., the MSE) and tune the $\lambda$ such that the conditional coverage can be maximized within the specified tolerance.

When $\lambda\rightarrow\infty$, the objective is equivalent to only minimize the regularization term. However, the convergence of the predictive set heavily depends on many factors such as the underlying data distribution and the non-conformity scores. In some cases, there may be infinitely many equally good solutions (in terms of the regularization term) with different predictive sets. 

To see why it is the case, we can take a look at our synthetic data setup II. In this toy example, any function $y\equiv c, c \in\mathbb{R}$ can minimize the regularization term $KS(P(V),P_\phi(V|X))$ when $P_\phi$ is perfect. All these solutions will have a perfect conditional coverage, which is the purpose of the regularization. 

However, if $c$ is far away from $0$, the predictive sets will be very wide. In fact, with different values of $c$, the sizes of the predict sets are generally different. Therefore, the final predictive set outputted will be highly dependent on the initial model parameters since there are many global optimums. However, all of these solutions will have similar conditional coverage when $\lambda\rightarrow\infty$ while we cannot control the set size since there is no penalty for wide confidence intervals. Future work can also consider to penalize the interval length in the objective function. 
In practice, the MSE term can be used to control the predictive accuracy of the function, and therefore control the set size of the predictive sets. 

\subsection{Selection of $\gamma$}

For the choice of $\gamma$, intuitively the $\gamma$ should not be too large or too small. If $\gamma$ is too small, the objective may be far from the indicator function we hope to approximate. If $\gamma$ is too large, the objective will be too similar to the discontinuous indicator function and bring challenges to the optimization problem. We empirically demonstrate it using our synthetic data. 

We plot the worst miscoverage rate across $\alpha$ with $\lambda=100$ under the synthetic data setup I with a $\gamma$ grid of $[1,3,5,7,9,11,20,50,70,100,200]$. The results are shown in \Cref{fig:miscoverage_gamma}. We can see the miscoverage is relatively high with very small and large values of $\gamma$ while the miscoverage rate is low and relatively stable when $\gamma$ is between 3 and 50, which justifies our choice of $\gamma=10$. 

\begin{figure}
    \centering
    \includegraphics[width=0.5\textwidth]{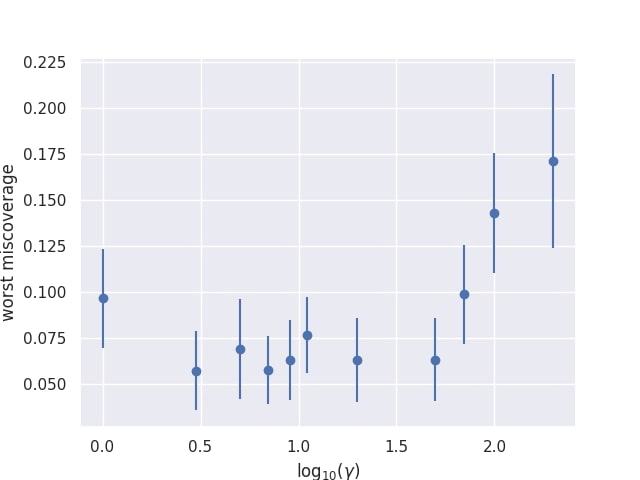}
    \caption{Worst Conditional Coverage vs $\log_{10}(\gamma)$. The conditional miscoverage rate is relatively stable for $\gamma$ between 3 and 50.}
    \label{fig:miscoverage_gamma}
\end{figure}

\section{Discussion on Computation Cost}\label{app:computation}

Compared to the standard conformal prediction method, our method requires learning a separate conditional density model. The computation costs for optimizing the conditional density model can vary heavily depending on the model class. 
For example, the cost of estimating the kernel density estimator is $O(mn)$, where $m$ is CDF grid size and $n$ is the sample size. The cost of training a deep conditional generative model can vary heavily depend on the model architecture and number of parameters. The computation time of estimating the KS distance is also $O(mn)$.

\section{Discussion on Adaptive Set Size Change}\label{app:adaptive}

First, we would like to note that our method can only lead to global set size change for residual and normalized score change. To see why, the set size of the residual non-conformity score is $2q^*$ and the set size of the normalized non-conformity score is $2\sigma(x)q^*$, where $q^*$ is the quantile of the marginal non-conformity score. By changing the function $f$ (our method), we can only influence the set size through $q^*$, which leads to a global change in the set size. 

However, our method can have an adaptive set size change for some non-conformity score such as quantile scores (the predictive set is $[\hat{q}_l(x)-q^*, \hat{q}_h(x)+q^*]$ and the set size is $\hat{q}_h(x)-\hat{q}_l(x)+2q^*$). To see an example, we segment the community dataset to black communities (the fraction of black residents is greater than 50\%) and non-black communities to check the sub-group coverage with the same experimental setup in the paper. The results are shown in \Cref{tab:adaptive}. Here KS-CP changes the set size adaptively and improves subgroup coverage. 

We have added the discussion above in the revised manuscript. 

\begin{table}[tbhp!]
    \centering
    \begin{tabular}{cccc}\toprule
         &  Black Community & Non-Black Community\\ \midrule 
       CP Set Size  &  2.91 & 2.81\\
       CP Coverage & 0.93 & 0.73 \\ \midrule 
       KS-CP Set Size & 2.56 & 2.21\\
       KS-CP Coverage & 0.92 & 0.79\\ \midrule 
       $\Delta$ Set Size & 0.35 & 0.60 \\ \bottomrule
    \end{tabular}
    \caption{Subgroup Set Size and Coverage with Quantile Non-Conformity Score. KS-CP changes the set size adaptively and improves subgroup coverage.}
    \label{tab:adaptive}
\end{table}

\end{appendices}

\FloatBarrier

\bibliography{sn-bibliography}

\end{document}